\documentclass{article}

\usepackage[nonatbib,preprint]{neurips_2020}
\usepackage{microtype}
\usepackage{adjustbox}
\usepackage{graphicx}
\usepackage[square,numbers]{natbib}
\usepackage{subfigure}
\usepackage{booktabs} % for professional tables
\usepackage{mathtools,amsmath,amssymb,amstext,amsthm,tikz,thmtools}
\usepackage[ruled,linesnumbered]{algorithm2e}
\usepackage{hyperref}       % hyperlinks
\usepackage{tikz}
\usepackage{url}            % simple URL typesetting
\usepackage{enumerate}
\usepackage{amsfonts}       % blackboard math symbols
\usepackage{nicefrac}       % compact symbols for 1/2, etc.
\usepackage[utf8]{inputenc} % allow utf-8 input
\usepackage[T1]{fontenc}    % use 8-bit T1 fonts
\usepackage{hyperref}       % hyperlinks
\usepackage{url}            % simple URL typesetting
\usepackage{booktabs}       % professional-quality tables
\usepackage{amsfonts}       % blackboard math symbols
\usepackage{nicefrac}       % compact symbols for 1/2, etc.
\usepackage{microtype}      % microtypography
\usepackage{multirow}
\usepackage{threeparttable}
\usepackage{amsmath}
\usepackage{enumitem}
\usepackage{ragged2e}
\usepackage{caption}
\usepackage{array}                             \usepackage{tikz}                            \usepackage{xr}
\usepackage{pgfplots}
\usepackage{pgfplotstable}

\newcommand{\mc}{\mathcal}
\newcommand{\mb}{\mathbf}
\newcommand{\tr}{\text{Tr}}

\newcommand\numberthis{\addtocounter{equation}{1}\tag{\theequation}}

\DeclareMathOperator*{\argmin}{arg\,min}

\DeclareMathOperator*{\vcdim}{vcdim}

\newtheorem{theorem}{Theorem}
\newtheorem{observation}{Observation}

\newtheorem{definition}[theorem]{Definition}

\newtheorem{corollary}[theorem]{Corollary}
\theoremstyle{definition}

\title{On sampling from data with duplicate records}

\author{%
  Alireza~Heidari \\
  Department of Computer Science\\
  University of Waterloo\\
  \texttt{a5heidar@uwaterloo.ca} \\
   \And
   Shrinu~Kushagra \\
  Department of Computer Science\\
  University of Waterloo\\
  \texttt{skushagr@uwaterloo.ca} \\
   \And
     Ihab F.~Ilyas \\
  Department of Computer Science\\
  University of Waterloo\\
  \texttt{ilyas@uwaterloo.ca} \\
}

\begin{document}

\maketitle

\begin{abstract}
    Data deduplication is the task of detecting records in a database that correspond to the same real-world entity.  Our goal is to develop a procedure that samples uniformly from the set of entities present in the database in the presence of duplicates. We accomplish this by a two-stage process. In the first step, we estimate the frequencies of all the entities in the database. In the second step, we use rejection sampling to obtain a (approximately) uniform sample from the set of entities. However, efficiently estimating the frequency of all the entities is a non-trivial task and not attainable in the general case. Hence, we consider various natural properties of the data under which such frequency estimation (and consequently uniform sampling) is possible. Under each of those assumptions, we provide sampling algorithms and give proofs of the complexity (both statistical and computational) of our approach. We complement our study by conducting extensive experiments on both real and synthetic datasets. 
\end{abstract}

\section{Introduction}
\label{submission}

In various domains, people rely on data to make critical decisions. For example, businesses use data to decide about operations, sales, and marketing; Hospitals maintain patient records to track their treatments; Governments keep a census of their population to determine various aspects of public policy. Due to the complexity of the acquisition system and various human errors, this data is often noisy. These errors also affect the statistical properties of the dataset, which compromises its utility in various analytics. 

One common type of error in a dataset is the presence of multiple records that correspond to the same real-world entity — often because the data is collected from multiple sources and curated by multiple teams of people \cite{heidari2019approximate}. The presence of duplicate records causes various problems for the downstream tasks. For example, consider the problem of estimating the mean and/or variance of some column of a table. The presence of duplicates can lead to inaccurate estimates \cite{livshits2020approximate}. Another example is unsupervised learning tasks, such as $k$-means clustering, where the presence of duplicates can perturb the computed centres and drastically change the clustering output. Similarly, in a supervised learning setting, the presence of duplicates changes the optimization function leading to undesired behaviour. This process of de-duplicating the data is also referred to as record linkage \cite{winkler1999state} or reference matching \cite{mccallum2000efficient}, or copy detection \cite{shivakumar1995scam}.

One common approach for removing the duplicates from a given dataset is compute a similarity (or distance) score for each pair of records, then clustering techniques are usually used to generate groups of records (a cluster represents duplicate records). Generating all record-pairs has a time complexity of $O(|X|^2)$ where $X$ denotes the dataset. To reduce the quadratic dependence on the dataset-size, \textit{blocking} techniques are used in the database community \cite{hernandez1995merge,bilenko2006adaptive,heidari2019holodetect,ananthakrishna2002eliminating}.  Locality-sensitive hashing are a class of methods, which partition a dataset into blocks such that similar points share a block  and dis-similar points are partitioned across different blocks with very high probability. The problem of detecting duplicates across the entire dataset now reduces to the problem of detecting duplicates across the blocks. Hence, the time complexity of these methods is $O(|X|^2/B)$ where $B$ is the number of blocks into which the data is partitioned. The quadratic time-complexity can still be computationally prohibitive for large-scale datasets that are common in the industry.  
 
 The complexity of these methods can still be prohibitive or unnecessary in many scenarios. Consider the following thought experiment, where we are given an unclean dataset $X$. Let $E$ be the set containing the distinct unique entities entities in $X$, and consider the task of estimating the mean of $E$. One approach is to use any of the previous methods to detect the duplicates in $X$, remove them to construct $E$, and compute the mean in $E$. However, for the purpose of estimation, full construction of $E$ is unnecessary. If we had a procedure $\mc A$, which generates a sample from $E$ (without constructing the full data set), we can use the sampled set $S$ to estimate the mean. Note that the above discussion is equally applicable to machine learning tasks such as classification, where the goal is to estimate the best classifier. Finding such a sampling algorithm $\mc A$, is the goal of this paper. We lay down the theoretical foundations for this framework. Formally,
 \begin{center}
 	\textit{Given an unclean dataset $X$ with duplicate entities, find a method $\mc A$ which can sample uniformly\\ from the clean version $E$ of the dataset (or the set of unique distinct entities of $X$)}. 
 \end{center}
 
 A previous approach, Sample-and-Clean \cite{wang2014sample} uses sampling approaches for data deduplication; if the frequency for all the entities $e \in E$ are known, the authors proposed a method to sample uniformly from the set $E$. However, the assumption that the frequencies are known is extremely restrictive. In almost all practical situations, we do not expect to have access to such information. In this paper, we propose the following two-stage approach. In the first stage, we estimate the frequency of all the entities from a small sample. In the second stage, we use these estimates to obtain a set sampled uniformly at random from $E$.  However, it is well-established that estimating frequencies (the first stage) is a non-trivial task \cite{diakonikolas2016learning,raskhodnikova2009strong}. A simple application of the fundamental theorem of learning shows that no such $\mc A$ (which is based on estimating frequencies) can exist in general for arbitrary datasets. Observation \ref{observation:impossibilityofsampling} (using Thm. 6.7 in \cite{shalev2014understanding}) asserts that to get a reliable estimate of the frequency for an entity with low frequency, we need a linear number of samples. 
\vspace{-8pt}
\begin{observation}[First impossibility result for sampling]
	\label{observation:impossibilityofsampling}
	Let $X = \{\enspace \overbrace{e_1, \ldots, e_1}^{\text{$f_1$ times} },  \enspace \overbrace{e_2, \ldots, e_2}^{\text{$f_2$ times} }\}$ be a dataset with two entities with frequencies $f_1$ and $f_2$ respectively such that $f_1 \le \sqrt{n}$ where $n = |X|$. Let $\mc A$ be any algorithm which receives a sample $S$ of size $m$ and tries to estimate the frequency of the entities $e_1$ and $e_2$. Let $\hat f_1, \hat f_2$ be the estimated frequencies. If $|\hat f_1 - f_1| \le f_1\epsilon$ with probability atleast $1 -\delta$, then we have that $m > C\frac{n \log(1/\delta)}{\epsilon^2}$
\end{observation}
%\vspace{-10pt}
In this paper, we investigate certain properties of the data under which it is possible to construct efficient (both statistical and computational) procedures that can sample uniformly at random form the set of entities $E$. We consider three categories of datasets/methods: (1) in Section \ref{section:databalance}, we consider datasets that are `balanced' (Defn. \ref{defn:etabalanced}) and show how that can help us estimate the frequencies of all the entities from a `small' sample; (2) in Section \ref{section:lshsampling}, we consider datasets that can be successfully partitioned into hashing blocks, and we show how access to such blocks can help us estimate the frequencies and then sample uniformly from the set of entities $E$; and (3) in section \ref{section:gaussiansampling}, we consider the case when the dataset is generated by a mixture of $k$-spherical Gaussian distributions. For all the three cases, we provide mathematical bounds to prove correctness of our approach. Finally, in Section \ref{section:experiments}, we provide extensive experimental evaluation of our approach on both synthetic and real-world datasets. Considering each method assumptions, we inject duplicate data into the real datasets and then we evaluate our methods by comparing the average of the output sample with the original data average.  
\begin{table}
\begin{adjustbox}{max width=1.01\textwidth}
\begin{tabular}{|c|c|c|c|c|}
\hline
Method&Goal& Assumptions & Sample Complexity & References \\ \hline
Goodman&Unbiased estimator for $|E|$ &$|E|\ll n$ &$\Omega\Big(|E|\Big)$&\cite{goodman1949estimation}\\
Valiant&Distribution support size&Large $|E|$ &$\Omega\Big(\frac{|E|}{\log|E|}\Big)$&\cite{valiant2011estimating}\\[7pt]
Sample-and-Clean&Unbiased \textit{avg} estimator &$f_1,\dots,f_{|E|}$ are given& NA&\cite{wang2014sample}\\[4pt]
Balanced Datasets &Uniform Sample&$\eta$-balance&$O\Big(\frac{1}{\epsilon^2\eta^2}\big(\log |E| \log \frac{\log |E|}{\epsilon\eta} + \log\frac{1}{\delta}\big)\Big)$&Thm \ref{thm:databalance}\&\ref{thm:databalancetime}\\
LSH-based&Uniform Sample&$\delta$-isotropic set&$O\Big(q\big[\log s + \log(\frac{2q}{\delta})\big]\Big)$&Thm \ref{thm:samplingclustering}\\
 Gaussian prior&Uniform Sample&Well-Separated GMM&$O\Big(\frac{d^3(\log k^2+\log\log\epsilon^{-1}+ \log\delta^{-1})}{\eta_{\min}\tau^2\epsilon^2}\Big)$&Thm \ref{thm:gmmsamplecomplexity}\\ \hline
\end{tabular}
\end{adjustbox}
\vspace{0.5em}
\caption{Bounds on the sample complexities of learning rejection process of generating uniform distribution. These results promise error of $\epsilon$ with probability $\delta$. $|E|$ determine the number of entities. $s$ is the maximum distance of lower and upper boundaries. $q$ is the number of blocks. $d$ is the dimension of Gaussians and $k$ determine the number of them. $\eta_{\min}$ is the smallest weight parameter.}
\vspace{-1.5em}
\end{table}

\label{tab:overview}

\justifying
In Table \ref{tab:overview}, we give an overview of all methods that we cover in this paper. The Gaussian prior method assumes a property of whole space and the generative process which is stronger than the LSH-based method. The LSH-based method assumes the boundary of each cluster is known. The balanced dataset method needs the lower-bound on the smallest cluster size, which is a weaker assumption than LSH-based method.

\section{Preliminaries and solution overview}

\justifying
We denote by $X$, the original unclean dataset. Let $E$ be the set of entities in $X$. That is, $E$ is the set of distinct elements(unique entities) in $X$. The frequency of an entity $e \in E$ is defined as $freq(e) = |x \in X: x = e|$ and the probability of an entity is defined as $prob(e) = \frac{freq(e)}{|X|}$. We denote by $\mc T_{X}$, the uniform distribution over the entities of $X$. In this paper, our goal is to sample (approximately) according to $\mc T_{X}$. Thus, we need a metric of distance between two distributions to quantify how far we are from our goal. For this, we use the total variation distance which is defined as $d_{TV}(\mc P, \mc Q) = \sup_{A \subseteq X} |\mc P(A) - \mc Q(A)|$

\begin{definition}[Cleanable]
	\label{defn:cleanable}
	Given set $X$ and parameters $\epsilon, \delta \in (0, 1)$. Let $E$ be the set of entities of $X$. We say that $X$ is cleanable if there exists an algorithm $\mc A$ and function $f$ such that we have the following. If $\mc A$ receives a sample $S$ of size $m \ge f(\epsilon, \delta)$ then with probability atleast $1-\delta$ (over the choice of $S$), $\mc A$ outputs a distribution $\mc P$ such that $d_{TV}(\mc P, \mc T_{X}) \le \epsilon$
\end{definition}
\justifying
If such an algorithm $\mc A$ exists for a dataset, we say that $\mc A$ cleans $X$. Furthermore, $X$  is cleanable with sample complexity given by $f$. Our general two-stage approach for constructing $\mc A$ is described below. Note that instead of outputting a distribution $\mc P$, we output a set $P$ of size $p$, sampled according to $\mc P$ \cite{devroye1986sample}. 
\vspace{-5pt}
\begin{algorithm}[h]
	\small
	\SetAlgoLined
	\KwIn{Dataset $X = \{x_1, \ldots, x_n\}$, sample size $p$}
	\KwOut{Sample $P$}
	
	\vspace{0.05in} \underline{First stage}\\
	Use a procedure $\mc F$, to estimate the probabilities (or frequencies) for all $e \in E$.\\
	Let $\hat {p}(e)$ be the estimated probabilities and let $m = \min \hat {p}(e)$\\

	\vspace{0.05in} \underline{Second stage}\\
	\While{ $|P|\neq p$}{
		Sample $v \in X$ uniformly at  random and let $a$ be a uniform random number in $[0, 1]$\\
		\If { $a < \frac{m}{\hat{p}(v)}$}{
			Add $v$ to $P$
		}
	}
	
	\vspace{0.05in}\textbf{return} $P$
	\caption{Uniform sampling from the clean data}
	\label{alg:genericcleaning}
\end{algorithm}

The second stage is a rejection sampling step, where we accept a point with probability inversely proportional to its (estimated) frequency. Thus, if the estimates are accurate, each point has an (approximately) equal probability of getting selected. The key component of our approach is the procedure $\mc F$ used during the first stage. Since, it is not possible to have an $\mc F$ in the general case, depending upon different properties of the dataset $X$, we use different procedures, as described in Section \ref{submission}. 
\vspace{-1em}
\paragraph{Case 1:  Balanced datasets} The data balance property asserts that the probability of each entity is atleast $\eta$. In Section \ref{section:databalance}, we describe a cleaning procedure for when the data has this property. 
\begin{definition}[$\eta$-balance]
\label{defn:etabalanced}
Given a set $X$ and the corresponding set of entities $E$. We say that $X$ is $\eta$-balanced w.r.t $E$ if $ \min_{e \in E} \enspace prob(e) \ge \eta$.
\end{definition}

\vspace{-1em}
\paragraph{Case 2:  Blocked datasets}

Next, let us consider the opposite spectrum for de-duplication applications; a common scenario described below where each entity has a small (atmost a constant) number of duplicates. To uniformly sample in the such scenarios, we turn our attention towards Locality Sensitive Hashing or LSH-based methods. LSH is a popular technique that aims to partition a given dataset (and an associated similarity or distance metric) into blocks such that two points whose similarity is above a certain threshold lie in the same block. \footnote{The actual definition says that two points lie in the same block with probability proportional to their similarity. But a non-probabilistic treatment suffices for this section.} A generic definition of LSH and related methods is given in the appendix section \ref{section:lsh}. In Section \ref{section:lshsampling}, we assume that the dataset has been partitioned into blocks by a suitable LSH-based method. We cluster each block using the framework of regularized $k$-means algorithm \cite{kushagra2017provably}. 

\begin{definition}[Regularized $k$-means objective]
	\label{defn:regularizedclustering}
	Given a clustering instance $(X, d)$ and the number of clusters $k$. Partition $X$ into $k+1$ subsets $\mc C = \{C_1, \ldots, C_k, C_{k+1}\}$ so as to minimize
	$\sum_{i=1}^k \sum_{x \in C_i} d^2(x, \mu_i) \enspace + \enspace \lambda|C_{k+1}|$.
	Here $\mu_i$ represents the center of $C_i$ where the cluster centres $\mu_1, \ldots, \mu_k$. In this framework, the algorithm is allowed to `discard' points into a garbage cluster $C_{k+1}$. 
\end{definition}
\vspace{-1em}
We then combine the clustering of each of these blocks into a clustering of the whole dataset. In Section \ref{section:lshsampling}, we further prove that if the dataset has the $\delta$-isotropic property (defined next), then our LSH and clustering based method cleans $X$. 
 
 \begin{definition}[$\delta$-isotropic set]
 	Let $\mc D$ be an isotropic distribution on the unit ball centred at the origin. Let $E = \{e_1, \ldots, e_n\}$ be points such that $\|e_i - e_j\| > \delta > 2$. Let $\mc D_i$ be the measure $\mc D$ translated w.r.t $e_i$. Let $X_i$ be a set of size $n_i$ generated according to the distribution $\mc D_i$. We say that $X = \cup X_i$ is a $\delta$-isotropic set and $E$ is the set of entities of $X$.
 \end{definition}
 \vspace{-1em}
 Some common example of isotropic distributions include standard Gaussian distribution, Bernoulli distribution, spherical distributions, uniform distribution and many more \cite{vershynin2010introduction}. 
 \vspace{-1em}
\paragraph{Case 3:  Gaussian prior}
Finally, in Section \ref{section:gaussiansampling}, we look at the same problem from a generative perspective. That is when the probability of each entity is well-approximated by a mixture of $k$ Gaussians (Defn. \ref{defn:xigmm}). 
 
 \begin{definition}[Well-Separated Gaussian mixture models]
 	For $i \in \{1, \ldots, k\}$, let $\mu_i, \Sigma_i$ be the parameters of $k$ different Gaussian distributions. Also, let the mixing weights $\eta_i \in [0, 1]$ be such that $\sum_i \eta_i = 1$. A mixture of $k$ Gaussians is well-separated if for all $i \neq j $, we have that $\|\mu_i - \mu_j\| \ge C\max(\sigma_i, \sigma_j)\sqrt{\log(\rho_\sigma/\eta_{min})} = O(\sqrt{\log k})$
 \end{definition}
\vspace{-1em}
We say that a database $X$ has $\xi$-gmm property if it can be `well-approximated' by a mixture of $k$ Gaussian distributions. We describe this intuition formally below.  
\begin{definition}[$\xi$-GMM property]
	\label{defn:xigmm}
	Given a finite dataset $X$. Let $\eta_i, \mu_i$ and $\sigma_i$ be the parameters of a well-separated mixture of $k$ spherical Gaussians with density function $\mc N$. We say that $X$ has $\xi$-GMM property  if for all $x \in X$, we have that 
	$|prob(x) - \mc N(x)| \enspace \le \enspace \mc N(x)\xi$
\end{definition}

\section{Sampling for balanced datasets}
\label{section:databalance}

\justifying
In this section, we consider datasets that satisfy the $\eta$-balanced property (Defn. \ref{defn:etabalanced}). The cleaning algorithm is as described in Alg. \ref{alg:genericcleaning}. Procedure $\mc F$ (which estimates the frequencies) works as follows. We first sample a set $T$ of size $m$ uniformly at random from $X$. We compute the count of all entities in our sample $T$ and use the counts (divided by $m$) as probability estimates for these sampled points. The detailed approach is  included in the appendix. (see Alg. \ref{alg:probbalance} in appendix). Thm. \ref{thm:databalance} establishes rigorous bounds on the approximation guarantees of our sampling procedure. It shows that the sampling distribution approximates the uniform distribution (where the distance between two distributions is measured by total variation distance). Thm. \ref{thm:databalancetime} analyses the time complexity of our approach and shows that the time taken to sample one point is constant in expectation.

\begin{theorem}
\label{thm:databalance}
Given a finite dataset $X = \{x_1, \ldots, x_n\}$ which satisfies $\eta$-balance property w.r.t its set of entities $E$. Let $\mc A$ be as described in Alg. \ref{alg:genericcleaning} with procedure $\mc F$ as described in Alg. \ref{alg:probbalance}. If $\mc F$ receives a sample of size $ m \ge f(\epsilon, \delta) := \frac{a}{\epsilon^2\eta^2}\Big(\log |E| \log \frac{\log |E|}{\epsilon\eta} + \log\frac{1}{\delta}\Big)$ then $\mc A$ cleans $X$ with sample-complexity given by the function $f$. 
\end{theorem}

\begin{theorem}
\label{thm:databalancetime}
	Let the framework be as in Thm. \ref{thm:databalance}. And define $\eta_1 = \max_{e \in E} prob(e)$ and $\eta_2 = \min_{e\in E} prob(e)$. Then the preprocessing time of algorithm $\mc A$ is $O(\log^2 |E|)$ and the expected time taken to sample one point is $O(\frac{\eta_1}{\eta_2})$.
	\vspace{-7pt}
\end{theorem}

The data-balance property implies that all the entities have a frequency of atleast $\eta |X|$. However, this assumption does not hold in all cases. As a statistical sufficient extreme example, consider a dataset which has no duplicates. For this dataset, $\eta = \frac{1}{|X|}$ and our sample complexity results are vacuous.

\section{LSH-based sampling}
\label{section:lshsampling}

In Section \ref{section:databalance}, we saw a method that samples approximately according to the uniform distribution (over the entities) if the given dataset has $\eta$-niceness. The number of samples required to construct this distribution is $O(\frac{1}{\eta^2}\log \frac{1}{\eta})$. In situations when $\eta = O(\frac{1}{n})$, the bounds from the previous section are vacuous. In this section, we assume that the data has been partitioned into hash-blocks such that all the duplicates are within the same block. That is for all $x \in X$, all $y$ which correspond to the same entity as $x$, share the same hash-block. Hence, we can treat each block as a separate instance of the cleaning (or the de-duplication) problem. 

Our goal is to estimate the frequency (or probability) of each entity in $E$. To achieve this goal, we cluster the set $X$ and then estimate the frequency of an entity $e$ as the number points that belong to the same cluster as $e$. Since each block can be treated independently, we focus on clustering a hash-block rather than the whole set. Clustering a hash block, although easier than clustering the entire dataset, still has some issues: the number of clusters is still unknown and hence standard clustering formulations are inapplicable for our setting.  In Section \ref{section:regularized}, we describe our regularized clustering algorithm which can cluster each hash block given $k$, the number of non-singleton clusters.

The exact knowledge of the number of non-singleton clusters for each hash block can still be restrictive in many applications. A weaker assumption is the knowledge of an upper and lower bound on the \textit{number of non-singleton clusters} within each hash-block. That is for each block, we know that $k \in [k_1, k_2]$ where $k$ is the number of non-singleton clusters and $k_1$ and $k_2$ are known. In Section \ref{section:ssc}, we describe a principled approach to select the right value of $k$ based on the framework of SSC (semi-supervised clustering) introduced in \cite{kushagra2019semi, kushagra2018semi} and describe our complete sampling approach.
\vspace{-8pt}
\subsection{Regularized $k$-means clustering}
\label{section:regularized}
\vspace{-8pt}
To solve the optimization problem in Defn.  \ref{defn:regularizedclustering}, we use the following strategy. We first decide which points go into the set $C_{k+1}$ (that is which points belong to singleton clusters). We remove those points from the set and $k$-cluster the remaining points using an SDP based algorithm (same as in \cite{kushagra2017provably}). Our approach is described in Alg. \ref{alg:regularizedclustering}. For space constraints, some of the details of our approach are included in the appendix section. 

\setlength{\textfloatsep}{5pt}
\begin{algorithm}[]
	\small
	\KwIn{Clustering instance $(X, d)$, the number of non-singleton clusters $k$ and constant $\mu$}
	\KwOut{Partition into $k+1$ clusters.}
	\label{alg:regularizedsdp}

	\vspace{0.1in} For all $x$, compute $S_x = \{y : d(x, y) \le \mu\}$. If $|S_x| > 1$ then $X' = X' \cup S_x$.\\
	$C_{k+1} = X \setminus X'$ and $X = X'$. \\
	If $|X| \le constant$, execute a brute force search for all possible $k$ partitions.\\
	Return the output of  Alg.1 in \cite{kushagra2017provably} with input $X, k$ and $\lambda = \infty$.	
	\caption{Regularized $k$-means clustering}
	\label{alg:regularizedclustering}
\end{algorithm}

\noindent Next, Thm. \ref{thm:regularizedsdp} shows that under $\delta$-isotropic assumption and if the number of non-singleton clusters $k$ is known, then Alg. \ref{alg:regularizedclustering} finds the desired clustering solution. 

\begin{theorem}
\label{thm:regularizedsdp}
Given a clustering instance $(X, d)$ where $x_i \in X$ has dimension $p$. Let $X$ be a $\delta$-isotropic set and let $E = \{e_1, \ldots, e_n\}$ be the set of entities of $X$. Let ${e_1, \ldots, e_k}$ be the set of non-singleton entities of $X$. In addition, let $e_i \in X$. Denote by $B_i$ all the records in $X$ which correspond to the entity $e_i$ and $C_{k+1} = \{e_{k+1}, \ldots, e_n\}$. If
$\delta > 2 + O( \sqrt{k/p})$
then there exists a constant $c > 0$ such that with probability at least $1 - 2p\exp(\frac{-cN\theta}{p\log^2N})$ Alg. \ref{alg:regularizedclustering} finds the intended cluster solution  $\mc C^* = \{B_1, \ldots, B_k, C_{k+1}\}$ when given $X, k$ and $\mu = 1$ as input. 
\end{theorem}
\vspace{-8pt}
\subsection{Semi-supervised clustering}
\label{section:ssc}
\vspace{-8pt}
In the previous section, we discussed an algorithm, which finds the target clustering when the number of non-singleton clusters $k$ is known. In this section, we extend it to the case when it is given that $k \in [k_1, k_2]$. We use the framework of semi-supervised clustering selection (SSC). 

\begin{definition}[Semi-Supervised Clustering (SSC) \cite{kushagra2019semi}]
	\label{defn:rcc}
	Given a clustering instance $(X, d)$. Let $C^*$ be an unknown target clustering of $X$. Find $\hat C \in \mc G := \{C_1, \ldots, C_p\}$ such that 
	$\hat C = \argmin_{C \in \mc G} \enspace L_{C^*}(C)$ where $L_{C^*}(C)$ measures the (weighted) average of the fraction of pairs of points which belong to the same= cluster according to $C^*$ but belong to different clusters in $C$ plus the fraction of pairs which belong to the different clusters in $C^*$ but belong to same cluster in $C$.
\end{definition}

For each value of $k$ from $k_1, \ldots, k_2$, we use Alg. \ref{alg:regularizedclustering}, to generate clusterings $\mc G = \{\mc C_{k_1}, \ldots, C_{k_2} \}$. Note the each $\mc C_{k_i}$ is a clustering of the given dataset. We then use the SSC framework to select the best clustering from $\mc G$. Owing to space constraints, we describe the details of the SSC algorithm (almost identical to the algorithm in \cite{kushagra2018semi}) and related proofs in the appendix section. We describe our ``clustering and hashing'' based sampling algorithm and then prove the main result from this section.

\setlength{\textfloatsep}{5pt}
\begin{algorithm}[]
	\small
	\SetAlgoLined
	\caption{Probability estimates for all entities under LSH}
	\label{alg:samplingclustering}
	\KwIn{Dataset $X$ which has been partitioned in to blocks $X_1, \ldots, X_q$ and sample size $m$}
	\KwOut{$\hat p$.}

	\While{ $1 \le i \le q$ } { 
	 let the number of non-singleton clusters $ \in [k_{i1}$, $k_{2i}]$.\\
	let $ \mc F_i = \phi$.\\
	\For {$k \in [k_{i1}, k_{i2}]$}{
		Use Alg. \ref{alg:regularizedclustering} with input $X_i, d, k$ to obtain clustering $\mc C$.\\
		Let $\mc G_i = \mc G_i \cup \mc C$.\\
		Use the SSC framework (Alg. \ref{alg:ERM} in appendix) with sample $\frac{m}{q}$ to obtain $\hat{\mc C_i}$ of $X_i$. 
	}	
	Combine the clusterings to obtain a clustering $\hat {\mc C}$ of the whole set $X$.
}
Define $\hat{p}(e) = \frac{1}{|\hat{\mc C}(e)|} $ where $C(e)$ denotes the number of points which belong to the same cluster as $e$.\\
\end{algorithm}

\begin{theorem}
\label{thm:samplingclustering}
Given a finite dataset $X = \{x_1, \ldots, x_n\}$ which has the $\delta$-isotropic property w.r.t its set of entities $E$. Let $x_i$ have dimension $g$. Let $X$ be partitioned into blocks $X_1, \ldots, X_q$ such that all records corresponding to the same entity lie within the same hash block. For each of the blocks $X_i$ let $k_i$ be the number of entities with number of corresponding records greater than $1$ (or non-singleton clusters). Let $\mc C_i^*$ be the corresponding clustering of the non-singleton entities of $X_i$ be such that any other clustering $C$ of $X_i$ has loss $L_{\mc C_i^*}(C) > o(\alpha)$. Let $\mc A$ be as described in Alg. \ref{alg:genericcleaning} with procedure $\mc F$ as described in Alg. \ref{alg:samplingclustering}. If $\mc F$ receives a sample of size $m \ge  a q \frac{\log s + \log(\frac{2q}{\delta})}{\alpha^2}$
where $a$ is a universal constant and $s = \max_i (k_{i2} - k_{i1})$ where $k_{i2}, k_{i1}$ are as defined in Alg. \ref{alg:samplingclustering}. Then with probability atleast $1-\delta - 2gq\exp(\frac{-cN\theta}{g\log^2N})$\footnote{$c$ is a global constant and $N = \min B_i$ where $B_i$ is the total number of points in non-single clusters for $1 \le i \le q$. The minimum is over all $B_i$ greater than a large global constant.}, $\mc A$ samples a set $P$ of size $p$ such that $$d_{TV}(\mc P, \mc T_{X}) = 0.$$
\end{theorem}

\section{Sampling under a Gaussian prior}
\label{section:gaussiansampling}

In this section, we consider finite datasets $X$, which have been generated by a distribution that has the $\xi$-GMM property. That is, the probabilities of all the entities can be well approximated by a mixture of $k$-Gaussian distributions. We use the popular EM algorithm to estimate the parameters of the mixture model. Next, we prove that if the generative model is a mixture of $k$ well-separated spherical Gaussians then the above approach samples a point approximately according to the uniform distribution.

\begin{algorithm}[t]
	\small
	\caption{Probability estimates for all entities under GMM prior}
	\label{alg:gaussianprior}
	\KwIn{Dataset $X \subseteq \mb R^d$, the number of mixtures $k$, sample size $m$ and number  of steps $T$}
	\KwOut{Sample $S$}
	
	\vspace{5pt}Run the EM algorithm for $T$ steps with $X$ as input and obtain parameters $\theta_i = (\hat \eta_i, \hat\mu_i, \hat \sigma_i)$.\\
	
	Define $\hat{\mc N}(x) = \sum_i \hat \eta_i \mc N (x; \hat\mu_i, \hat \sigma_i)$ where $\mc N(x; \mu, \sigma)$ is the Gaussian with mean $\mu$ and variance $\sigma^2$.
\end{algorithm}
\begin{theorem}
\label{thm:gmmsamplecomplexity}
	Given a finite dataset $X$ which has the $\xi$-GMM property w.r.t an unknown density function $\mc N$ with parameters $\eta_i, \mu_i$, $\sigma_i$ and $\tau=\min \mc N$. Let $E$ be the set of entities of $X$.Let $\mc A$ be as described in Alg. \ref{alg:genericcleaning} with procedure $\mc F$ as described in Alg. \ref{alg:gaussianprior}. If $\mc F$ receives a sample of size $m > C'\frac{d^3(\log(k^2T)+ \log(\frac{1}{\delta}))}{\eta_{\min}\tau^2\epsilon^2}$ and $T=O(\log(\frac{1}{\tau\epsilon}))$ as input, then $\mc A$ samples a set $P$ according to a distribution $\mc P$ such that $$d_{TV}(\mc P(e), \mc T_{X}) \le \epsilon + \xi$$
	with probability atleast $1 -\delta - \frac{T}{k^{30}n^{C-2}}$ where $C$ is the separation parameter for the spherical Gaussians. Note, we assume that the parameters for the EM algorithm are initialized as in Thm. \ref{thm:kwon} (in the appendix). 
\end{theorem}

We see that the sample complexity depends inversely on $\tau$. This shows that the sampling gets progressively more difficult for distributions that have a long tail. Also, we see that $\xi$ introduces a bias in our estimates. The smaller the $\xi$, the better we estimate the true frequencies (or probabilities). 

\section{Experimental Results}
\label{section:experiments}

In this section, experiments have been divided into two parts, the experiments that show behaviours of our framework and experiments compare our estimator to Sample-and-Clean \cite{wang2014sample} on some real datasets. We describe the datasets, metrics, and experimental settings used to validate our estimator in appendix. We determine $Error = \vert RealAvg-EstimatedAvg\vert/RealAvg$, and evaluate different method over our datasets. We repeated each experiment until, we see convergence in the average of the errors.
\vspace{-8pt}
\subsection{Effect of Sampling Size and Dataset Balance}
\vspace{-8pt}
We evaluate our sampling method for balance dataset under different sample sizes and perform this evaluation for different duplication ratios. For this experience, we use \textit{TPC-H} dataset and we inject duplicated values manually. Table \ref{fig:pvs} shows that for different duplication rates, the method has a similar behaviour and the error decreases as the sample size increases, the error is strictly smaller than the theoretical upper bound. In Table \ref{fig:npvs}, we generate an arbitrary distribution for entities frequencies. From Table \ref{fig:pvs} and Table \ref{fig:npvs}, as Thm \ref{thm:databalance} suggests, we confirm that the imbalance dataset weaken the uniform sample generation.

\begin{minipage}{.5\linewidth}
\hspace{-0.6em}
\begin{threeparttable}
\begin{tabular}{l|c c c c c c}{}
\textit{dup} & 0.01& 0.02 & 0.04& 0.06& 0.08 & 0.1\\ \hline \hline
\multirow{1}{*}{0.1}   & 2.12\tnote{\#} & 1.64 & 1.41& 1.23 & 1.11 & 1.16\\ \hline
\multirow{1}{*}{0.15}  & 3.12 & 2.03 & 1.84 & 1.92 & 1.76 & 1.69\\ \hline
 \multirow{1}{*}{0.2}  & 3.14 & 2.24 & 1.97 & 1.84 & 1.80 & 1.74\\ \hline
  \multirow{1}{*}{0.25}  & 4.26 & 4.02 & 3.65 & 2.93 & 2.61 & 2.17\\ \hline
  \multirow{1}{*}{0.3}  & 5.21 & 4.94 & 4.57 & 3.84 & 3.33 & 2.89 \\ \hline
\end{tabular}
\begin{tablenotes}
  \item[\#] Values $\times 10^{-3}$.
  \end{tablenotes}
\end{threeparttable}
\captionsetup{width=.9\linewidth}
\captionof{table}{ The precision changes in different sample sizes under generative process for duplication with uniform distribution. By increasing the duplication ratio, the error of our framework increases. \textit{dup} presents duplication rate.}
\label{fig:pvs}
\end{minipage}%
\begin{minipage}{.5\linewidth}
\hspace{0.4em}
\begin{threeparttable}
\begin{tabular}{l|cccccc}{}
\textit{dup} &0.01&0.02 & 0.04& 0.06& 0.08 & 0.1\\ \hline \hline
\multirow{1}{*}{0.1}   & 2.58\tnote{\#} & 2.03 & 1.69& 1.38 & 1.26 & 1.19\\ \hline
\multirow{1}{*}{0.15}  & 3.66 & 3.17 & 2.50 & 2.03 & 1.89 & 1.79\\ \hline
\multirow{1}{*}{0.2}  & 4.62 & 4.14 & 3.61 & 3.08 & 2.59 & 2.42\\ \hline
\multirow{1}{*}{0.25}  & 5.32 & 4.88 & 4.37 & 3.84 & 3.29 & 2.73\\ \hline
\multirow{1}{*}{0.3}  & 5.94 & 5.45 & 4.99 & 4.73 & 4.03 & 3.79 \\ \hline
\end{tabular}
\begin{tablenotes}
  \item[\#] Values $\times 10^{-3}$.
  \end{tablenotes}
\end{threeparttable}
\captionsetup{width=.9\linewidth}
\captionof{table}{Our estimator is independent from duplication distribution of entities. The datasets that considered has non-uniform duplication over their entities.\vspace{2.1em}}
\label{fig:npvs}
\end{minipage}
\vspace{-8pt}
\subsection{Our Methods Over Real Datasets}
\vspace{-8pt}
We conducted a set of experiments on real datasets to evaluate our method and evaluate our theoretical bounds. We set $\delta=0.9$ and obtain all information each method needs directly from data. For each sample size, we repeat for $100$ times and calculate the average of the errors.
Fig \ref{exp:realbalances} shows the result of the Alg. \ref{alg:genericcleaning} for four real datasets, and the dashed line is linear regression of the upper-bound suggested by Thm. \ref{thm:databalance}. Figure \ref{exp:reallsh} shows the result of Alg. \ref{alg:samplingclustering} and the dashed lines show the upper bound given by Thm. \ref{thm:samplingclustering}. In Fig \ref{exp:realgmm}, we used the Alg \ref{alg:gaussianprior}, and computed the upper bound by using Thm. \ref{thm:gmmsamplecomplexity}. As we know the assumption of Gaussian prior is stronger than LSH and LSH is stronger that balanced dataset. Gaussian method on a random dataset has weaker performance, which can approved by comparing Fig \ref{exp:realgmm} with Fig \ref{exp:reallsh} and Fig \ref{exp:realbalances}.

\begin{figure}
\centering
\begin{minipage}{.33\textwidth}
\centering
\begin{tikzpicture}[scale=0.55]
\begin{axis}[ymin=0.0, ymax=0.27,xmax=0.53, legend pos=north east,
xlabel={\textbf{Sample Size}}, ylabel={\textbf{Error}}]
\addplot [line width=0.7mm,mark=*] coordinates
{ (0.05,0.089229936) (0.1,0.07242393678) (0.15,0.05625649963) (0.2,0.04068371) (0.3,0.016457087) (0.4,0.00643613) (0.5,0.00064417574)};
\addlegendentry{Publications}
\addplot [line width=0.7mm,mark=*,color=blue] coordinates
{ (0.05,0.1289229936) (0.1,0.1202393678) (0.15,0.1045649963) (0.2,0.0883710697) (0.3,0.0545741087) (0.4,0.0320361326) (0.5,0.0086417574)};
\addlegendentry{Product I}
\addplot [line width=0.7mm,mark=*,color=red] coordinates
{ (0.05,0.10473457) (0.1,0.0962883953) (0.15,0.088284385) (0.2,0.05628924797) (0.3,0.0325741087) (0.4,0.0086361326) (0.5,0.00320417574)};
\addlegendentry{Product II}
\addplot [line width=0.7mm,mark=*,color=green] coordinates
{ (0.05,0.1273457) (0.1,0.1203953) (0.15,0.09604385) (0.2,0.089924797) (0.3,0.0705741087) (0.4,0.042361326) (0.5,0.0129483)};
\addlegendentry{Restaurants}
\addplot [dotted, color=black,line width=0.7mm] table[mark=none,row sep=\\,y={create col/linear regression={y=Y}}]
    {
        X Y\\
        0.05 0.219229936\\
        0.1 0.192393678\\
        0.15 0.145649963\\
        0.2 0.103710697\\
        0.3 0.065741087\\
        0.4 0.040361326\\
        0.5 0.054417574\\
    };
\addplot [dotted, color=blue,line width=0.7mm] table[mark=none,row sep=\\,y={create col/linear regression={y=Y}}]
    {
        X Y\\
        0.05 0.229229936\\
        0.1 0.19242378\\
        0.15 0.165649963\\
        0.2 0.143710697\\
        0.3 0.115741087\\
        0.4 0.100361326\\
        0.5 0.084417574\\
    };
\addplot [dotted, color=red,line width=0.7mm] table[mark=none,row sep=\\,y={create col/linear regression={y=Y}}]
    {
        X Y\\
        0.05 0.249229936\\
        0.1 0.21242378\\
        0.15 0.195649963\\
        0.2 0.143710697\\
        0.3 0.105741087\\
        0.4 0.070361326\\
        0.5 0.074417574\\
    };
\addplot [dotted, color=green,line width=0.7mm] table[mark=none,row sep=\\,y={create col/linear regression={y=Y}}]
    {
        X Y\\
        0.05 0.239229936\\
        0.1 0.25242378\\
        0.15 0.195649963\\
        0.2 0.183710697\\
        0.3 0.175741087\\
        0.4 0.110361326\\
        0.5 0.104417574\\
    };
\end{axis}
\end{tikzpicture}
\captionsetup{width=.9\linewidth}
\captionof{figure}{Applying balanced method on real datasets. Dotted lines show the linear regression of theoretical bounds.}
\label{exp:realbalances}

\end{minipage}%
\begin{minipage}{.33\textwidth}
  \centering
\begin{tikzpicture}[scale=0.55]
\begin{axis}[ymin=0.0, ymax=0.28,xmax=0.53, legend pos=north east,
xlabel={\textbf{Sample Size}}, ylabel={\textbf{Error}}]
\addplot [line width=0.7mm,mark=*] coordinates
{ (0.05,0.1289229936) (0.1,0.112393678) (0.15,0.1045649963) (0.2,0.08883710697) (0.3,0.0405741087) (0.4,0.032361326) (0.5,0.0088417574)};
\addlegendentry{Publications}
\addplot [line width=0.7mm,mark=*,color=blue] coordinates
{ (0.05,0.1449229936) (0.1,0.1282393678) (0.15,0.1125649963) (0.2,0.07283710697) (0.3,0.0245741087) (0.4,0.016361326) (0.5,0.008417574)};
\addlegendentry{Product I}
\addplot [line width=0.7mm,mark=*,color=red] coordinates
{ (0.05,0.1673457) (0.1,0.15883953) (0.15,0.12884385) (0.2,0.088924797) (0.3,0.0485741087) (0.4,0.024361326) (0.5,0.012847574)};
\addlegendentry{Product II}
\addplot [line width=0.7mm,mark=*,color=green] coordinates
{ (0.05,0.1273457) (0.1,0.10403953) (0.15,0.0804385) (0.2,0.0488924797) (0.3,0.0245741087) (0.4,0.009361326) (0.5,0.00169483)};
\addlegendentry{Restaurants}
\addplot [dotted, color=black,line width=0.7mm] table[mark=none,row sep=\\,y={create col/linear regression={y=Y}}]
    {
        X Y\\
        0.05 0.279229936\\
        0.1 0.222393678\\
        0.15 0.20019963\\
        0.2 0.163710697\\
        0.3 0.135741087\\
        0.4 0.120361326\\
        0.5 0.094417574\\
    };
\addplot [dotted, color=blue,line width=0.7mm] table[mark=none,row sep=\\,y={create col/linear regression={y=Y}}]
    {
        X Y\\
        0.05 0.279229936\\
        0.1 0.24249178\\
        0.15 0.2253563\\
        0.2 0.193710697\\
        0.3 0.165741087\\
        0.4 0.130361326\\
        0.5 0.114417574\\
    };
\addplot [dotted, color=red,line width=0.7mm] table[mark=none,row sep=\\,y={create col/linear regression={y=Y}}]
    {
        X Y\\
        0.05 0.319229936\\
        0.1 0.23242378\\
        0.15 0.195649963\\
        0.2 0.183710697\\
        0.3 0.155741087\\
        0.4 0.120361326\\
        0.5 0.104417574\\
    };
\addplot [dotted, color=green,line width=0.7mm] table[mark=none,row sep=\\,y={create col/linear regression={y=Y}}]
    {
        X Y\\
        0.05 0.239229936\\
        0.1 0.22242378\\
        0.15 0.185649963\\
        0.2 0.153710697\\
        0.3 0.135741087\\
        0.4 0.100361326\\
        0.5 0.044417574\\
    };
\end{axis}
\end{tikzpicture}
\captionsetup{width=.9\linewidth}
\captionof{figure}{Applying LSH method on real datasets. Dotted lines show the linear regression of theoretical bounds.}
\label{exp:reallsh}
\end{minipage}%
\begin{minipage}{.33\textwidth}
   \centering
\begin{tikzpicture}[scale=0.55]
\begin{axis}[ymin=0.0, ymax=0.33,xmax=0.53, legend pos=north east,
xlabel={\textbf{Sample Size}}, ylabel={\textbf{Error}}]
\addplot [line width=0.7mm,mark=*] coordinates
{ (0.05,0.179229936) (0.1,0.158393678) (0.15,0.1565649963) (0.2,0.14483710697) (0.3,0.096741087) (0.4,0.04840361326) (0.5,0.021417574)};
\addlegendentry{Publications}
\addplot [line width=0.7mm,mark=*,color=blue] coordinates
{ (0.05,0.149229936) (0.1,0.122393678) (0.15,0.0955649963) (0.2,0.08583710697) (0.3,0.0655741087) (0.4,0.050361326) (0.5,0.010417574)};
\addlegendentry{Product I}
\addplot [line width=0.7mm,mark=*,color=red] coordinates
{ (0.05,0.17673457) (0.1,0.165683953) (0.15,0.144484385) (0.2,0.1138924797) (0.3,0.06145741087) (0.4,0.015361326) (0.5,0.0043417574)};
\addlegendentry{Product II}
\addplot [line width=0.7mm,mark=*,color=green] coordinates
{ (0.05,0.12473457) (0.1,0.09403953) (0.15,0.084404385) (0.2,0.0648924797) (0.3,0.03675741087) (0.4,0.0123361326) (0.5,0.00429483)};
\addlegendentry{Restaurants}
\addplot [dotted, color=black,line width=0.7mm] table[mark=none,row sep=\\,y={create col/linear regression={y=Y}}]
    {
        X Y\\
        0.05 0.319229936\\
        0.1 0.302393678\\
        0.15 0.305649963\\
        0.2 0.273710697\\
        0.3 0.245741087\\
        0.4 0.140361326\\
        0.5 0.104417574\\
    };
\addplot [dotted, color=blue,line width=0.7mm] table[mark=none,row sep=\\,y={create col/linear regression={y=Y}}]
    {
        X Y\\
        0.05 0.319229936\\
        0.1 0.28242378\\
        0.15 0.265649963\\
        0.2 0.2013710697\\
        0.3 0.175741087\\
        0.4 0.140361326\\
        0.5 0.114417574\\
    };
\addplot [dotted, color=red,line width=0.7mm] table[mark=none,row sep=\\,y={create col/linear regression={y=Y}}]
    {
        X Y\\
        0.05 0.299229936\\
        0.1 0.27242378\\
        0.15 0.225649963\\
        0.2 0.193710697\\
        0.3 0.155741087\\
        0.4 0.110361326\\
        0.5 0.0904417574\\
    };
\addplot [dotted, color=green,line width=0.7mm] table[mark=none,row sep=\\,y={create col/linear regression={y=Y}}]
    {
        X Y\\
        0.05 0.249229936\\
        0.1 0.23242378\\
        0.15 0.215649963\\
        0.2 0.203710697\\
        0.3 0.175741087\\
        0.4 0.110361326\\
        0.5 0.044417574\\
    };
\end{axis}
\end{tikzpicture}
\captionsetup{width=.9\linewidth}
\captionof{figure}{Applying GMM method on real datasets. Dotted lines show the linear regression of theoretical bounds.}
\label{exp:realgmm}
\end{minipage}
\end{figure}
\vspace{-8pt}
\subsection{Comparison to Other Methods Over Real Data}
\vspace{-8pt}
In this section, we compare out methods to \textit{RawSC} and \textit{NormalizedSC} in Sample-and-Clean \cite{wang2014sample} on real and synthetic datasets. We use samples with size $30\%$ of the dataset, and measure $accuracy=1-error$. We use the optimal blocking function for \textit{RawSC} and \textit{NormalizedSC}.

\begin{figure}[ht]

\centering
\includegraphics[width=\columnwidth]{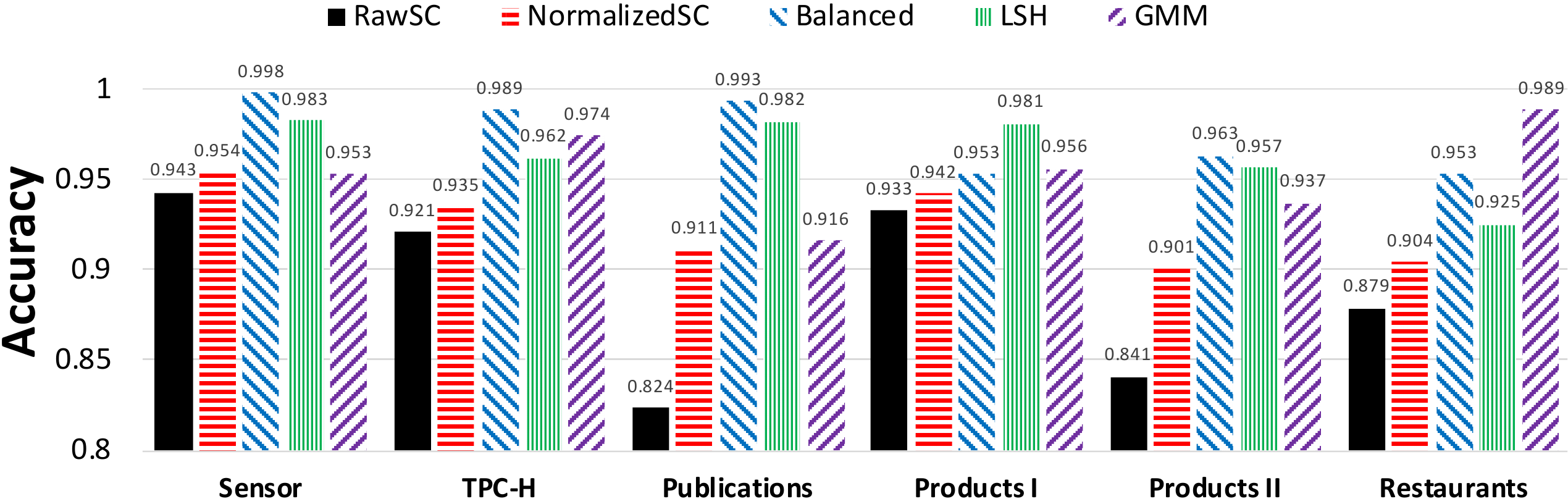}
\caption{$accuracy$ of proposed methods and methods in \cite{wang2014sample}. In Sensor and TPC-H, duplicated records are added manually.}
\label{fig:compare}
\end{figure}

As you can see in Figure \ref{fig:compare}, $accuracy$ of our methods outperforms the the-state-of-art sample cleaning framework. In each method, we computed the assumed information as the method input. For each experiment, we computed the average of $100$ experiments.

\section{Conclusion and Open Problems}
Obtaining correct information from data with repetitive records is an important problem \cite{heidari2020record}. Deduplicating the entire dataset is computationally expensive. We solve the problem by approximating the uniform distribution over the clean data. Generating uniform sample from such a data is not always possible. Knowing additional data properties can make the problem feasible and solvable. We consider three approaches that work under different circumstances. These methods return samples that can be used for any downstream analytical purpose, because it has the same properties as a uniform sample from the cleaned version of the data. There are some open-ended questions in our research. One direction of research can explore other weaker and/or natural assumptions under which the problem is still solvable. Another direction of research involves providing tighter bounds and/or lower bounds for the methods presented in this paper. Another important question, is if it possible to verify that a given dataset satisfies any of the niceness conditions. For example, is it possible to estimate the value of $\eta$ for a balanced dataset (Section \ref{section:databalance}). In the appendix, we have introduced a method which provides a lower bound estimate for the balance parameter $\eta$. However, the number of samples needed (upper and lower bound) to obtain this estimate is an open question. Similar questions can be posed for the other niceness conditions introduced in this paper.

\section{Broader Impact}
The ability to provide a clean sample from large unclean data sets will solve two main problems in large-scale data cleaning: (1) avoiding the prohibitively costly process of cleaning large data sets to perform simple analytic tasks, which often requires a small clean sample; and (2) significantly accelerating the development of data prep solutions, which use clean sample to choose model parameters and solution settings. Our continuous discussions with large enterprises, through startups dedicated to data cleaning, revealed that the ability to effectively produce clean samples from large unclean data repositories is a major bottleneck in making these data sets available in data science pipelines. We believe that the paper addresses this important problem, and opens the door to more follow-up works that further relaxes the initial assumptions made in this proposal.

\vspace{1em}
{\fontsize{15}{100}\selectfont \textbf{Appendix}}
\begin{appendix}

\section{Sampling for balanced datasets}

\begin{algorithm}[h]
	\small
	\textbf{Input: }Dataset $X = \{x_1, \ldots, x_n\}$\\
	\textbf{Output: } Probability estimates $\hat p_1, \ldots, \hat p_E$ for all the entities in $X$. 
	
	\vspace{0.1in}Let $W = \{v_1,v_2,\dots,v_m\}$ be a set of size $m$ sampled uniformly at random from $X$.\\
	For all $v_i \in W$, let $\hat {p}(v_i) := \frac{|\{w\in W:   x = v_i\}|}{m}$. \\
	For all $v_i \in X \setminus W$, let $\hat{p}(v_i) = \min_{w \in W} \hat{p}(w)$\\
	
	\textbf{return} $\hat p$
	\caption{Probability estimates of all the entities for balanced datasets}
	\label{alg:probbalance}
\end{algorithm}

\begin{theorem}
	\label{thm:databalance}
	Given a finite dataset $X = \{x_1, \ldots, x_n\}$ which satisfies $\eta$-balance property w.r.t its set of entities $E$. Let $\mc A$ be as described in Alg. \ref{alg:genericcleaning} with procedure $\mc F$ as described in Alg. \ref{alg:probbalance}. If $\mc F$ receives a sample of size $$ m \ge f(\epsilon, \delta) := \frac{a}{\epsilon^2\eta^2}\Big(\log |E| \log \frac{\log |E|}{\epsilon\eta} + \log\frac{1}{\delta}\Big)$$
	then $\mc A$ cleans $X$ with sample-complexity given by the function $f$. 
\end{theorem}

\begin{proof}
	Let $m, W$ and $\mc A$ be as defined in the description of Alg. \ref{alg:probbalance}. Let $h_x = \{ x_i \in X: x_i = x\}$. And let $H = \{h_x : x \in X\}$ be a set of subsets of $X$. Now, $|H| = |E| = r$. Hence, we get that the $\vcdim(H) \le  \log r$. Now, using the classical result from learning theory (Thm. \ref{thm:vceapprox}), we get that if $$m \ge \frac{a}{\epsilon^2}\Big(d \log \frac{d}{\epsilon} + \log\frac{1}{\delta}\Big) =: M$$
	where $d = \log r$ is an upper-bound on the $\vcdim (H)$, then with probability atleast $1-\delta$, we have that for all $h_x \in H$, $\Big|\frac{|h_x \cap W|}{|W|} - p(h_x)\Big| \le \epsilon$. Now, denote by $q(x) := \frac{|h_x \cap W|}{|W|}$. Then, we get that for $m \ge M$, with probability $1-\delta$, for all $x \in X, \big|q(x) - p (x)\big| \le \epsilon$. Now, for all $x \not\in W$, we have that $p(x) \le \epsilon$ which contradicts the fact that $p(x) > \eta$. Thus, we see that the sampling procedure samples a point with probability $q(x) \propto \frac{p(x)}{\hat p(x)}$ where $1 - \frac{\epsilon}{\eta-\epsilon} \le \frac{p(x)}{\hat p(x)} \le 1 + \frac{\epsilon}{\eta -\epsilon}$
	Choosing $\epsilon = \frac{\epsilon}{(1+\epsilon)}\eta$ gives us the result of the theorem.
\end{proof}

\begin{theorem}
	Let the framework be as in Thm. \ref{thm:databalance}. And define $\eta_1 = \max_{e \in E} prob(e)$ and $\eta_2 = \min_{e\in E} prob(e)$. Then the algorithm $\mc A$ has the following properties.
	\begin{itemize}[nolistsep,noitemsep]
		\item The preprocessing time is $O(\log^2 |E|)$.
		\item The expected time taken to sample one point is $O(\frac{\eta_1}{\eta_2})$.
	\end{itemize}
\end{theorem}
\begin{proof}
	Note that the preprocessing time depends linearly on $m$. Using the bound on $m$ from Thm. \ref{thm:databalance} the result on pre-processing time follows. Similarly, the expected time taken to sample a point is upper bounded by $\frac{\eta_1 + \epsilon}{\eta -\epsilon}$ where $\epsilon$ is as defined in the proof of Thm. \ref{thm:databalance}. Substituting the values of $\epsilon = \frac{\epsilon}{(1+\epsilon)}\eta $, we get that the expectation is upper bounded by $\epsilon+ (1+\epsilon)\frac{\eta_1}{\eta_2}$.
\end{proof}
\subsection{Approximating $\eta$ with Sampling}
 , and under the assumption of $\eta$-balanced, we proved balanced dataset method. In this section, we give an analysis that can approximate $\eta$ using the given sample $T$. Let $m$ be the sample size. We know the fact that in discrete distribution $\frac{1}{n}\leq \frac{\min_i |E_i|}{|m|}\leq \frac{1}{|E|}$ and $\frac{1}{|E|}\leq \frac{\max_i |E_i|}{|m|}\leq 1-\frac{|E|}{n}$. This means the distribution of values in $\frac{|E|}{m}$ is dependent on $\eta$. Let $f_i$ be the number of values that appear exactly $i$ times in $T$ and denote by $r$ the number of distinct values in the sample, so $\sum_{i=1}^{m} f_i = r$ and $\sum_{i=1}^{m} if_i = m$.

\begin{theorem}
Given a the sample $S = \{v_1,v_2,\dots,v_m\}$. Denote the values of $\frac{|E|}{n}$ with $\mathbf{c}$, then we have the following inequality

$$\eta \geq \frac{1}{\hat E}-(1-\frac{1}{\hat E})\sigma_c\sqrt{2r}$$
where $$\hat E = r+\sum\limits_{i=1}^{m}(-1)^{i+1}\frac{(n-m+i-1)!(m-i)!}{(n-m-1)!m!}f_i~~,~~ \sigma_c=\sqrt{ \frac{1}{r}\sum_{c\in \mathbf{c}} c^2 - \bigg(\frac{1}{r}\sum_{c\in \mathbf{c}} c\bigg)^2}.$$
\end{theorem}{}

\begin{proof}
$\eta$-niceness provide a lower bound of values on $\mathbf{c}$. Using this fact, we can determine the variance of $\mathbf{c}$. If we know the minimum of a distribution is $\eta$, The smallest possible maximum would be $\frac{1-\eta}{|E|-1}$. From Szőkefalvi Nagy inequality \cite{nagy1918algebraische} if maximum value be $Max$ and minimum value be $Min$, we have the following

$$\frac{1}{r}\sum_{c\in \mathbf{c}} c^2 - \bigg(\frac{1}{r}\sum_{c\in \mathbf{c}} c\bigg)^2 \geq \frac{(Max-Min)^2}{2r}$$
We denote the left side of the inequality with $\sigma_c^2$. $Max$ is unknown, but $Max \geq \frac{1-\eta}{|E|-1}$, so we have $\sigma_c^2 \geq \frac{(\frac{1-\eta}{|E|-1}-\eta)^2}{2r}$. We can obtain $\eta\geq \frac{1}{|E|}-(1-\frac{1}{|E|})\sigma_c\sqrt{2r}$. We use Goodman estimator \cite{goodman1949estimation} to obtain an unbiased estimation of $|E|$. 
\begin{equation}
\label{goodman}
  \hat E = r+\sum\limits_{i=1}^{m}(-1)^{i+1}\frac{(n-m+i-1)!(m-i)!}{(n-m-1)!m!}f_i  
\end{equation}

\end{proof}
Note that $\eta$ is fixed in the given dataset, so we can use a  parallel sample to approximate $\eta$. 

\section{LSH-based sampling}
\begin{definition}[Hash function]
	Given a set $X$. A hash function $h: X \rightarrow \{1, \ldots, k\}$ partitions the set $X$ into $k$ blocks. 
\end{definition}

\begin{definition}[LSH]
	\label{defn:LSH}
	Given a set $X$ and a similarity measure $s: X\times X \rightarrow [0, 1]$. Let $\mc H$ be a set of functions over $X$. An LSH for the similarity measure $s$ is a probability distribution over $\mc H$ such that for all $x_1, x_2 \in X$
	$$\underset{h \in \mc H}{\mb P}[h(x_1) = h(x_2)] = s(x_1, x_2)$$
\end{definition}

\subsection{Locality Sensitive Hashing}
\label{section:lsh}
Note that the above definition is terms of similarity function $s$. For our purposes, it will be more comfortable to talk in terms of the distance metric $d$, rather than a similarity function. Note that a given distance metric implies a similarity function and vice-versa. Hence, given $(X, d)$, if there exists a such a probability distribution, then we say that the given metric $d$ is LSH-able.  

We are now ready to describe a generic hashing scheme based on a LSH-able metric. Sample hash functions $h_1, \ldots, h_k$ and group them into $k$ groups $H_1, \ldots, H_r$ of size $s$ each, that is, $rs = k$. Now, two points $x_1, x_2$ end in the same block if they have same hash value on either one of the $k$ groups.  The approach is described below.

\begin{algorithm}
	\caption{A generic LSH based hashing algorithm \cite{indyk1998approximate,charikar2002similarity}}
	\label{alg:LSH}
	
	\Indp\KwIn{$(X, d)$, a class of hash functions $\mc H$ and integers $r, s$.}
	\KwOut{Partition $Q$ of the set $X$.}
	\vspace{0.1in} Let $D$ be a distribution over $H$ which satisfies Defn. \ref{defn:LSH} and let $k = r s$.\\
	Sample hash functions $h_1, \ldots, h_k$ iid using $D$.\\
	Group the hash functions into $s$ bands. Each band contains $r$ hash functions. \\
	For all $x$ and $0\le i \le s-1$,  let $g_i(x) = (h_{is + 1}(x), \ldots, h_{is+r}(x))$. That is, $g_i(x)$ represents the $i^{th}$ signature of $x$. \\
	Let $Q$ be the partition induced by $g_i$'s. That is, if there exists $0 \le i < s$ such that $g_i(x_1) = g_i(x_2)$ then $x_1$ and $x_2$ belong to the same group in $Q$. \\
	Output $Q$.
\end{algorithm}

\begin{theorem}
	Given a set $X$, a distance function $d: X \times X \rightarrow [0, 1]$, a class of hash functions $H$, threshold parameter $\lambda$ and a parameter $\delta$. Let $\mc A$ be a generic LSH based algorithm (Alg. \ref{alg:LSH}) 
	
	\noindent Choose $r, s$ such that $\frac{1}{2\lambda} < r < \frac{1}{-\ln(1-\lambda)} $ and $s =  \lceil 2.2\ln(\frac{1}{\delta})\rceil$. Define $\delta' := s\ln(1+\delta)$. Then for $x_1, x_2 \in X$
	\begin{itemize}
		\item If $d(x_1, x_2) \le \lambda$ then $\mb P_{h \in H} \enspace[\enspace q(x_1, x_2) = 1 ] \enspace > \enspace 1 - \delta$
	\end{itemize}
	where $q(x, y) = 1$ iff $x, y$ belong to the same group in $Q$.
\end{theorem}
\begin{proof}
	Observe that
	\begin{align*}
	& \mb P[b(x, y) = 0] = \mb P \enspace  [ \underset{i=1}{\cap^s} g_i(x) \neq g_i(y)] =\prod_i\Big(1 - \prod_{j=1}^r \mb P[ h_{(i-1)r+j}(x) = h_{(i-1)r+j}(y)]\Big)\\
	& = \prod_{i=1}^s(1 - \prod_{j=1}^r f(x, y)) = (1-f(x, y)^r)^s
	\end{align*}
	Consider the case when $d(x, y) \le \lambda$. From the choice of $s$, we know that $s \ge 2.2\ln(1/\delta) \implies s \ge \frac{\ln(1/\delta)}{1-\ln(e-1)}\iff 1-\frac{1}{e} \le \delta^{1/s}$. From the choice of $r$, we know that $r < \frac{1}{-\ln(1-\lambda)} \iff r \ln(\frac{1}{1-\lambda}) < 1 \iff (1-\lambda)^r > \frac{1}{e}$.  Hence,   then we have that $\mb P[b(x, y) = 0]  = (1 - (1-d(x, y))^r)^s \le (1 - (1-\lambda)^r)^s < \delta$.
\end{proof}

\noindent For the simplicity of analysis, for the rest of the subsections, we will assume that if $x_1, x_2 \in X$ are duplicates of one another then $q(x_1, x_2) = 1$. In the probabilistic case our results hold true with the corresponding probability. 

\subsection{Regularized $k$-means clustering}
The algorithm is described in detail here.
\begin{algorithm}
	\caption{Regularized $k$-means clustering}
	\label{alg:regularizedsdp}
	\KwIn{Clustering instance $(X, d)$, the number of non-singleton clusters $k$ and constant $\mu$}
	\KwOut{Partition into $k+1$ clusters.}
	
	\vspace{0.1in} For all $x$, compute $S_x = \{y : d(x, y) \le \mu\}$. If $|S_x| > 1$ then $X' = X' \cup S_x$.\\
	$C_{k+1} = X \setminus X'$ and $X = X'$. \\
	If $|X| \le constant$, execute a brute force search for all possible $k$ partitions.\\
	For all $x_i \in X$, compute the matrix $D_{ij} = \|x_i-x_j\|^2_2$.\\
	Set $\lambda = \infty$ and $y=0$ and solve Eqn. \ref{eqn:regularisedSDP} using any standard SDP solver and obtain matrix $Z$.
	\begin{equation}
	\textbf{ SDP } 
	\begin{cases}
	\min_{Z, y} \enspace &\tr(DZ) + \lambda \langle \mb 1, y\rangle\\
	\text{s.t. } \enspace &\tr(Z) = k\\
	& \Big(\frac{Z+Z^T}{2}\Big)\cdot \mb 1 + y = \mb 1\\		
	&Z \ge 0, y \ge 0, Z \succeq 0 \numberthis\label{eqn:regularisedSDP}
	\end{cases}
	\end{equation}
	
	$k$-cluster the columns of $X^TZ$ to obtain clusters $C_1, \ldots, C_k$.\\
	
	Output $\mc C' = \{C_1, \ldots, C_k, C_{k+1}\}$.
\end{algorithm}

\begin{theorem}
	\label{thm:regularizedsdp}
	Given a clustering instance $(X, d)$ where $x_i \in X$ has dimension $p$. Let $X$ be a $\delta$-isotropic set and let $E = \{e_1, \ldots, e_n\}$ be the set of entities of $X$. Let ${e_1, \ldots, e_k}$ be the set of non-singleton entities of $X$. In addition, let $e_i \in X$. Denote by $B_i$ all the records in $X$ which correspond to the entity $e_i$ and $C_{k+1} = \{e_{k+1}, \ldots, e_n\}$. If
	$$\delta > 2 + O\Bigg( \sqrt{\frac{k}{p}} \Bigg)$$  
	then there exists a constant $c > 0$ such that with probability at least $1 - 2p\exp(\frac{-cN\theta}{p\log^2N})$ Alg. \ref{alg:regularizedsdp} finds the intended cluster solution  $\mc C^* = \{B_1, \ldots, B_k, C_{k+1}\}$ when given $X, k$ and $\mu = 1$ as input. 
\end{theorem}
\begin{proof}
	We know that $X$ satisfies $\delta$-isotropic condition. Hence, for all the entities $e \in U := \{e_{k+1}, \ldots, e_n\}$, we have that $|S_e| = 1$. Also, for $e \in U$ we have that $e \not\in S_{e'}$ (cause of $\delta$-isotropy). Thus, we have that $e \not\in X'$ and $e \in C_{k+1}$. Hence, we get that $U \subseteq C_{k+1}$. 
	
	Now, we will show that all $x \in X \setminus U$, doesn't belong to $C_{k+1}$. For the sake of contradiction, assume that there exists $x \in X \setminus U$ such that $x \in C_{k+1}$. WLOG, let $x \in e_i$ where $1 \le i \le k$. This implies that for all $x' \in X$, $x \not\in S_x'$. This is a contradiction as $x \in S_{e_i}$.  Thus, we get that $U = C_{k+1}$. Now, using Thm. 5.7 from \cite{kushagra2019theoretical} completes the proof of the theorem. 
\end{proof}

\subsection{Semi-supervised clustering}
\label{section:ssc}
In the previous section, we discussed an algorithm which finds the target clustering when the number of non-singleton clusters $k$ is known. In this section, we extend it to the case when it is given that $k \in [k_1, k_2]$. We use the framework of semi-supervised clustering selection (SSC) introduced in \cite{kushagra2019semi}. 

\begin{definition}[Clustering loss]
	Given a clustering $C$ of a set $X$ and an unknown target clustering $\mc C^*$. Denote by $P^+$ the uniform distribution over $\{(x, y) \in X^2: C^*(x, y) = 1\}$ and $P^-$ the uniform distribution over $\{(x, y) \in X^2: C^*(x, y) = 0\}$. The loss of clustering $C$ is defined as 
	\begin{align*}
	&L_{C^*}(C) = \mu \enspace  \underset{(x, y) \sim P^+}{\mb P} \big[ C(x, y) = 0 ] + (1-\mu)\enspace \underset{(x, y) \sim P^-}{\mb P} \big[ C(x, y) = 1 ]
	\end{align*}
	where $C(x, y) = 1$ iff $x, y$ belong to the same cluster according to $C$.
\end{definition}

\begin{definition}[Semi-Supervised Clustering selection (SSC)]
	\label{defn:rcc}
	Given a clustering instance $(X, d)$. Let $C^*$ be an unknown target clustering of $X$. Find $\hat C \in \mc G := \{C_1, \ldots, C_p\}$ such that 
	\begin{align}
	\hat C = \argmin_{C \in \mc G} \enspace L_{C^*}(C)\label{eqn:RCCMain}
	\end{align}
\end{definition}

For each value of $k$ from $k_1, \ldots, k_2$, we use Alg. \ref{alg:regularizedclustering}, to generate clusterings $\mc G = \{\mc C_{k_1}, \ldots, C_{k_2} \}$. Note the each $\mc C_{k_i}$ is a clustering of the given dataset. We then use the SSC framework to select the best clustering from $\mc G$. Next, we describe our SSC algorithm. This is a standard empirical risk minimization. We approximate the loss of all the clusterings from a sample and then choose one with minimum empirical loss. 

\begin{algorithm}[h]
	\caption{Empirical Risk Minimization for SSC}
	\label{alg:ERM}
	\KwIn{$( X, d)$, a set of clusterings $\mc F$, a $C^*$-oracle and size $m$.}
	\KwOut{$ C \in \mc F$}
	
	\vspace{0.1in} Sample a pair $(x, y)$ uniformly at random from $X^2$. If $C^*(x, y) = 1$ then $S_+ = S_+ \cup (x, y)$ else $S_- = S_- \cup (x, y)$. \\
	Repeat till at least one of $|S_+|$ and $|S_-|$ is less than $m$.\\
	Define $\hat{pl}(C) = \frac{|\{(x, y) \in S_+: C(x, y) = 0\}|}{|S_+|}$ and $\hat{nl}(C) = \frac{|\{(x, y) \in S_-: C(x, y) = 0\}|}{|S_-|}$
	
	Define $\hat L(C) = \mu \hat{pl}(C) + (1-\mu)\hat{nl}(C)$. \label{algLine:alpha}\\
	Output $\argmin_{C \in \mc F} \enspace \hat L(C)$ 
\end{algorithm}

\begin{theorem}[Sample Complexity]
	\label{thm:sampleComplexity}
	Given metric space $(X, d)$, a class of clusterings $\mc F$ of size $s$ and a threshold parameter $\lambda$. Given $\epsilon, \delta \in (0, 1)$ and a $C^*$-oracle. Let $\mc A$ be the ERbased approach as described in Alg. \ref{alg:ERM} and $\bar C$ be the output of $\mc A$. Let $C^* \in \mc F$. If  
	\begin{align}
	&m \enspace \ge a\frac{\log s + \log(\frac{2}{\delta})}{\epsilon^2} 
	\end{align}
	where $a$ is a global constant then with probability at least $1-\delta$ (over the randomness in the sampling procedure), we have that $$L_{C^*}(\bar C) \enspace\le\enspace \epsilon$$
\end{theorem}
\begin{proof}
	The proof of the theorem involves a straightforward application of the fundamental theorem of learning. If $m > a\frac{\vcdim(\mc F) + \log(\frac{1}{\delta})}{\epsilon^2}$, then with probability at least $1-\delta$, we have that $|\hat{nl}(C) - nl(C)| < \epsilon$. Similarly, we have that with probability at least $1-\delta$, we have that $|\hat{pl}(C) - pl(C)| < \epsilon$. Combining these two equations, we get that with probability at least $1-2\delta$, $|\hat{l}(C) - l(C)| < \epsilon$. Now, $l(\bar C) \le \hat l(\bar C) + \epsilon \le \hat l (C^*) + \epsilon \le l(C^*) + 2\epsilon$. Substituting, $\delta = \delta/2$ and $\epsilon = \epsilon/2$ completes the result of the theorem. 
\end{proof}

\noindent Next we prove an upper bound on the number of queries to the oracle to sample $m_+$ positive and $m_-$ negative pairs.
\begin{theorem}[Query Complexity]
	\label{thm:queryComplexity}
	Let the framework be as in Thm. \ref{thm:sampleComplexity}. In addition, let $\gamma = \mb P[C^*(x, y) = 0]$. With probability at least $1-\exp\big(-\frac{\nu^2m_-}{4}) - \exp\big(-\frac{\nu^2m_+}{4}\big)$ over the randomness in the sampling procedure, the number of same-cluster queries $q$ made by $\mc A$ is  
	$$q \le (1+\nu)\bigg(\frac{m_-}{\gamma} + \frac{m_+}{1-\gamma}\bigg)$$
\end{theorem}
\begin{proof}
	Let $q_-$ denote the number queries to sample the set $S_-$. Now, $\mb E[q_-] = \frac{1}{\gamma}$. Thus, using Thm. \ref{thm:geometricRV}, we get that $q_- \le \frac{(1+\nu)m_-}{\beta(1-\epsilon)}$ with probability at least $1-\exp(\frac{-\nu^2m_-}{4})$.
\end{proof}

\subsection{Putting it all together}
\begin{theorem}
	\label{thm:samplingclustering}
	Given a finite dataset $X = \{x_1, \ldots, x_n\}$ which has the $\delta$-isotropic property w.r.t its set of entities $E$. Let $x_i$ have dimension $g$. Let $X$ be partitioned into blocks $X_1, \ldots, X_q$ such that all records corresponding to the same entity lie within the same hash block. For each of the blocks $X_i$ let $k_i$ be the number of entities with number of corresponding records greater than $1$ (or non-singleton clusters). Let $\mc C_i^*$ be the corresponding clustering of the non-singleton entities of $X_i$ be such that any other clustering $C$ of $X_i$ has loss $L_{\mc C_i^*}(C) > o(\alpha)$. 
	
	Let $\mc A$ be as described in Alg. \ref{alg:genericcleaning} with procedure $\mc F$ as described in Alg. \ref{alg:samplingclustering}. If $\mc F$ receives a sample of size $m \ge  a q \frac{\log s + \log(\frac{2q}{\delta})}{\alpha^2}$ where $a$ is a universal constant and $s = \max_i (k_{i2} - k_{i1})$ where $k_{i2}, k_{i1}$ are as defined in Alg. \ref{alg:samplingclustering}. Then with probability at least $1-\delta - 2gq\exp(\frac{-cN\theta}{g\log^2N})$\footnote{$c$ is a global constant and $N = \min B_i$ where $B_i$ is the total number of points in non-single clusters for $1 \le i \le q$. The minimum is over all $B_i$ greater than a large global constant.}, $\mc A$ samples a set $P$ of size $p$ such that $$d_{TV}(\mc P, \mc T_{X}) = 0.$$
\end{theorem}

\begin{proof}
	Let $m$ be as in the statement of the theorem. Then, Thm. \ref{thm:sampleComplexity} implies that with probability at least $1-\delta$, we have that for all $i$, $L_{\mc C^*_i}(\hat{\mc C_i}) \le \epsilon$. However, we know that for all clusterings $\hat{\mc C_i}$ of $X_i$, we have that $L_{\mc C^*_i}(\hat{\mc C_i}) > \epsilon$. Hence, $\hat{\mc C_i} = \mc C^*_i$. Hence, we get that $\hat C = C^*$. In other words, Alg. \ref{alg:samplingclustering} recovers the target clustering. Once the target clustering is known, the rest of the algorithm samples a point uniformly at random and accepts it with probability proportional to $\frac{1}{|C(x)|}$ where $C(x)$ denotes the cluster to which $x$ belongs. Hence, for any entity $e$, we have that $$\mb P[e] \propto \frac{|C^*(e)|}{|C(e)|} = 1.$$ The extra $2gq\exp(\frac{-cN\theta}{g\log^2N})$ term is due to the success probability of the regularized sdp algorithm. 
\end{proof}

\section{Sampling under Gaussian prior}
\begin{theorem}
	Given a finite dataset $X$ which has the $\xi$-GMM property w.r.t an unknown density function $\mc N$ with parameters $\eta_i, \mu_i$, $\sigma_i$ and $\tau=\min \mc N$. Let $E$ be the set of entities of $X$. Let $\mc U$ be the uniform distribution over $E$. Given $m$ and $T=O(\log(1/\tau\epsilon))$ as input, Alg. \ref{alg:gaussianprior} samples a set $S$ according to a distribution $\mc P$ such that for all $e \in E$, we have that $$|\mc P(e) - \mc U(e)| \le \epsilon+ \xi$$
	Note, we assume that the parameters for the EM algorithm are initialized as in Thm. \ref{thm:kwon}. 
\end{theorem}

\begin{proof}
	Let functions $p$ and $\mc N$ be as defined in Defn. \ref{defn:xigmm}. Let $\hat{\mc N}$ be as in Alg. \ref{alg:gaussianprior} and $\eta_i, \mu_i, \sigma_i$ and $\hat\eta_i, \hat\mu_i, \hat\sigma_i$ be the parameters of the Gaussians $\mc N$ and $\hat{\mc N}$ respectively.
	
	From the description of Alg. \ref{alg:gaussianprior}, we see that each entity is sampled with probability $\mc P(e) =c \frac{p(e)}{\hat{\mc N}(e)}$. Here, $c$ is a constant such that $\sum_e \mc P(e) = 1$. Moreover, using Defn. \ref{defn:xigmm}, we have that 
	\begin{align}
	&(1-\xi) c \frac{\mc N(e)}{\mc {\hat N}(e)}\le \mc P(e) \le (1+\xi) c \frac{\mc N(e)}{\mc {\hat N}(e)}. \label{eqn:emProofEqXi}
	\end{align}
	Hence, we focus try to bound the term $\frac{\mc N(e)}{\mc {\hat N}(e)}$ below.  Now, using the definition of spherical Gaussians and the result from Thm. \ref{thm:kwon} (in the appendix), we have that
	\begin{align*}
	\small
	\mc N(e) &= \sum_i \frac{\eta_i}{(2\pi \sigma_i^2)^{\frac{d}{2}}} \exp\bigg( \frac{- \|x-\mu_i \|^2}{2\sigma_i^2} \bigg) \le \sum_i \frac{(1+\frac{\epsilon}{\sqrt{d}})^{\frac{d}{2}}(1+\epsilon)\hat\eta_i}{(2\pi \hat\sigma_i^2)^{\frac{d}{2}}} \exp\bigg( \frac{- (1-\frac{\epsilon}{\sqrt d})\|x-\mu_i \|^2}{2\hat\sigma_i^2} \bigg)\\
	& = (1+O(\epsilon)) \sum_i \frac{\hat\eta_i}{(2\pi \hat\sigma_i^2)^{\frac{d}{2}}} \exp\bigg( \frac{-(1-\epsilon')\|x-\mu_i \|^2}{2\hat\sigma_i^2} \bigg) \text{ where $\epsilon' = \frac{\epsilon }{\sqrt{d}}$}
	\end{align*}
	Now using triangle inequality, we have that $|\|x-\mu_i\| - \|x-\hat\mu_i\| | \le  \|\hat\mu_i-\mu_i\| \le \sigma_i \epsilon$.  Hence, $\|x-\mu_i\|^2 \ge \|x-\hat\mu_i\|^2 + \sigma_i^2 \epsilon^2 - 2\sigma_i\epsilon\|x-\hat\mu_i\|$. Substituting this in the above equation, we get that
	\begin{align*}
	\small
	\mc N(e) &\le (1+O(\epsilon)) \sum_i \frac{\hat\eta_i}{(2\pi \hat\sigma_i^2)^{\frac{d}{2}}} \exp\bigg( \frac{-(1-\epsilon')(\|x-\hat\mu_i \|^2 + \sigma_i^2\epsilon^2 - 2\sigma_i\|x-\hat\mu_i\| \epsilon )}{2\hat\sigma_i^2} \bigg)\\
	& = (1+O(\epsilon)) \sum_i \frac{\hat\eta_i}{(2\pi \hat\sigma_i^2)^{\frac{d}{2}}} \exp\bigg( \frac{-\|x-\hat\mu_i \|^2}{2\hat\sigma_i^2} \bigg)e^{\frac{\epsilon' \|x-\hat\mu_i\|^2   - \sigma_i^2\epsilon^2(1-\epsilon') + 2\sigma_i\|x-\hat\mu_i\|\epsilon(1-\epsilon')}{2\hat\sigma_i^2}}\\
	& = (1+O(\epsilon)) \sum_i \frac{\hat\eta_i}{(2\pi \hat\sigma_i^2)^{\frac{d}{2}}} \exp\bigg( \frac{-\|x-\hat\mu_i \|^2}{2\hat\sigma_i^2} \bigg)e^{\frac{(1+2\epsilon')(\epsilon' \|x-\hat\mu_i\|^2   - \sigma_i^2\epsilon^2(1-\epsilon') + 2\sigma_i\|x-\hat\mu_i\|\epsilon(1-\epsilon'))}{2\sigma_i^2}}\\
	& = (1+O(\epsilon))e^{O(\epsilon)} \hat{\mc N}(e). \enspace \text{ If $\epsilon\approx 0$, we have that}\\
	\frac{\mc N(e)}{\hat{\mc N}(e)} &\le  1 + O(\epsilon)
	\end{align*}
	The proof of the other direction is identical and is left as an exercise for the reader. Thus, we asymptotically get that $1 - O(\epsilon) \le \frac{\mc N(e)}{\hat{\mc N}(e)} \le 1 + O(\epsilon)$. For the case $\epsilon^2\approx 0, \epsilon\not\approx 0$, we use the result of Thm \ref{thm:samplecomplexity}, we have $\int_{b_{s}}^{b_e}\vert\mathcal{N}(s)-\mathcal{N'}(s)\vert ds\le\epsilon$ then because of well-separation property $b_e-b_{s}\ge 1$ and the bound $|\hat\eta_i - \eta_i| \le \eta_i \epsilon$, we have $1 - O(\epsilon/\tau) \le \frac{\mc N(e)}{\hat{\mc N}(e)} \le 1 + O(\epsilon/\tau)$, where $\tau= min N(x)$. Combining this with Eqn. \ref{eqn:emProofEqXi},	we get that $(1 - \xi)(1-O(\epsilon/\tau))c \le \mc P(e) \le (1+\xi)(1+O(\epsilon/\tau))c$. Hence, we get that $|\mc P(e) - \mc U(e)| \le \xi + O(\epsilon/\tau)$. Replacing $\epsilon = \epsilon/\tau$ gives us the result of theorem.
\end{proof}

\begin{theorem}[Thm 3.6 in \cite{kwon2020algorithm}] 
	\label{thm:kwon}
	Given a well-separated mixture of $k$-spherical Gaussians. There exists initializations for $\mu_1^{(0)}, \ldots, \mu_k^{(0)}$ for the means and $\eta_1^{(0)}, \ldots, \eta_k^{(0)}$ for the mixing weights such that if the EM algorithm is initialized with these parameters, and if each step of the EM algorithm receives a sample of size $m > C'\frac{d(\log(k^2T)+ \log(\frac{1}{\delta}))}{\eta_{\min}\epsilon^2}$ then in $T = O(\log(1/\epsilon))$ iterations, converges to parameters $\hat\eta_i, \hat\mu_i$ and $\hat\sigma_i$  such that for all $i$, 
	$$\|\hat \mu_i - \mu_i\| \le \sigma_i \epsilon \enspace\text{ and }\enspace |\hat\eta_i - \eta_i| \le \eta_i \epsilon \enspace\text{ and }\enspace |\hat \sigma_i^2 - \sigma_i^2|  \le \sigma_i^2 \epsilon/\sqrt d$$
	with probability at least $1 -\delta - \frac{T}{k^{30}n^{C-2}}$
\end{theorem}

\section{Mixture Model Generative Process}
In this section, we use mixture models (MM) to model the data generation process. We consider data generated with the mixture of $K$ spherical Gaussian mixtures with parameters $\{(\eta_k,\pmb\mu_k,\pmb\Sigma_k),\forall k\in [K]\}$. We are given $m$ observations $\pmb T = (\pmb x_1, . . . , \pmb x_m)$ form this $r$-dimensional space with allocation $\pmb e=\{e_1,\dots,e_m\}$ where $e_i$ is  a $K$-dimensional vector indicating to which component $\pmb x_i$ belongs, such that $e_{ij}\in \{0,1\}$ and $\sum_{j=1}^K e_{ij}=1$. From previous established result, we know that separation of $\Omega(\sqrt{\log k})$ is necessary and sufficient for identifiability of the parameters with polynomial sample complexity \cite{regev2017learning}, so we assume our Gaussians have such configuration.

The estimated parameters are denoted by $\hat{(.)}$.  We use $\mathcal{D}$ to represent the distribution of the mixture of Gaussians $G$, and $\mathcal{D}_k$ to represent the distribution of the $k^{\text{th}}$ Gaussian component. We assume each $\pmb x_i\in \pmb T$ drawn randomly from $\mathcal{D}$ with density $f(\pmb x|\pmb\theta)$ indexed by a parameter
$ \pmb\theta\in\Theta$.
\begin{equation}
f(\pmb x|\pmb\nu=(\pmb\theta_1, . . . , \pmb \theta_K,\pmb \eta))=\sum_{k=1}^K \eta_k f_{m}(\pmb x; \pmb \theta_k)
\end{equation}

where $\pmb\theta_k=(\pmb\mu_k,\pmb\Sigma_k)$ and $\pmb\Sigma_k\equiv \sigma^2_k\pmb I_r$. The problem is the component parameters $\pmb \theta_1, . . . , \pmb \theta_K$ and the weight distribution $\pmb \eta = (\eta_1, . . . , \eta_K)$ is unknown and these $3K-1$ parameters need to be estimated from the data. Parameter estimation could be based on the fully categorized data $(\pmb T, \pmb e)$ using standard methods of statistical inference, such as maximum likelihood estimation or Bayesian estimation. Using $\pmb\nu$, we can assign a new observation $\pmb x_{new}$ to a certain component $[K]$ using maximum a posteriori likelihood (MAP) and also we can use rejection probability $p(\pmb x_{new}|\pmb\nu)$ to obtain an almost uniform sample. We use maximum likelihood estimation to estimate the mixture model. 

\subsection{Complete-Data Maximum Likelihood Estimation}
\label{bayesian}
The EM algorithm is composed of two steps, the E-step that constructs the expectation of the log-likelihood on the current estimators, and the step that maximizes this expectation. For $\mathcal{D}$, we have the following algorithm. Algorithm \ref{alg:em} runs until MM parameters converge with error $\epsilon$.

\begin{algorithm}[h]
	\small
	\textbf{Input: }Dataset $X = \{\pmb x_1, \ldots,\pmb x_N\}$, Sample $\pmb T = (\pmb x_1, . . . , \pmb x_m)$, Allocation $\pmb e=\{e_1,\dots,e_m\}$, Number of distinct value $K$, Threshold $\epsilon$\\
	\textbf{Output: } GMM estimation\\
	
	\vspace{0.05in}
	Repeat until the total parameters change is less than $\epsilon$:
	
	Weak labeling of each observation $\pmb x_i$, for $i = 1, . . . ,m$, for each $k\in [K]$: 
	\begin{align*}
	\text{E-step}  & l_{k}= \frac{\eta_k e^{\frac{-\|\pmb x_i - \pmb\mu_k\|^2}{2\sigma_k^2}-\frac{d}{2}\log(\sigma_k^2)}}{\sum_{j=1}^K \eta_j e^{\frac{-\|\pmb x_i - \pmb\mu_j\|^2}{2\sigma_j^2}-\frac{d}{2}\log(\sigma_j^2)}}\\
	\text{M-step}  & \eta_k^+ = \mathbb{E}_{\pmb T}[l_k], \pmb \mu_k^+ =\frac{\mathbb{E}_{\pmb T}[l_k\pmb x_i ]}{\mathbb{E}_{\pmb T}[l_k]},\sigma_k^{+2}=\frac{\mathbb{E}_{\pmb T}[l_k\|\pmb x_i - \pmb\mu_k^+\|^2 ]}{d\mathbb{E}_{\pmb T}[l_k]}
	\end{align*}
	
	\vspace{0.1in}
	\textbf{return} $\hat f(\pmb x|\eta_k,\pmb \mu_k,\sigma^2_k)=\sum_{k=1}^K \eta_k \mathcal{N}(\pmb x; \pmb \mu_k,\sigma^2_k\pmb I_r)$
	\caption{EM estimation}
	\label{alg:em}
\end{algorithm}
In the above notation, we use $\mathbb{E}_{\pmb T}[.]$ to denote the expectation over the entire
sample of mixture distribution. In the E-step, $l_k$ represents the probability of the sample $\pmb x_i$ being generated from the $k^{\text{th}}$ component as computed using the current estimates of parameters, and $(.)^+$  denotes the corresponding updated estimators.

\subsection{Uniform Sample Using Estimated Mixture}
We use $\hat f(\pmb x|\eta_k,\pmb \mu_k,\sigma^2_k)$ obtained from Algorithm \ref{alg:em} to specific rejection probabilities and make a uniform sample from $G$. In Theorem \ref{thm:gibbs}, we prove that the returned sample is almost uniform.

\begin{theorem}
	\label{thm:gibbs}
	The sample $S$ returned in Algorithm \ref{alg:em} is uniform in $E$.
\end{theorem}
\begin{proof}
	We know that if $v$ is a random variable whose support is a subset of $[0,1]$, and $a$ is a standard uniform random variable independent of $v$, then $Pr(a\leq v)=\mathbb{E}[v]$. The probability of acceptance of $v$ is,
	\begin{align*}
	&Pr(v\text{ is accepted})=Pr\big(a\leq \frac{\min_{\pmb x\in E}f(\pmb x|\pmb{\nu})}{f(v|\pmb\nu)}\big)=\\&
	\mathbb{E}\big[\frac{\min_{\pmb x\in E}f(\pmb x|\pmb{\nu})}{f(v|\pmb\nu)}\big]=\min_{\pmb x\in E}f(\pmb x|\pmb{\nu})\int_{-\infty}^{+\infty}\frac{1}{f(x|\pmb\nu)}f(x|\pmb\nu)dx=K.\gamma
	\end{align*}
	The sampling procedure described produces draws from $E$ with density uniform. We must show that the conditional distribution of $v$ given that $a \leq \frac{\gamma}{p(v|\pmb\nu)}$ ,is indeed uniform in $E$;
	\begin{align*}
	Pr(&x\leq v|a \leq \frac{\gamma}{f(v|\pmb\nu)})=\frac{Pr(a \leq \frac{\gamma}{f(v|\pmb\nu)}|x\leq v).Pr(x\leq v)}{Pr(a \leq \frac{\gamma}{f(v|\pmb\nu)})}=\\& Pr(a \leq \frac{\gamma}{f(v|\pmb\nu)}|x\leq v).\frac{F(v)}{K.\gamma}=\frac{F(v)}{K.\gamma}.\frac{Pr(a \leq \frac{\gamma}{f(v|\pmb\nu)},x\leq v)}{F(v)}=\\&\frac{1}{K.\gamma}\int_{-\infty}^{v}Pr(a \leq \frac{\gamma}{f(v|\pmb\nu)},w\leq v)f(w|\pmb\nu)dw\\&\frac{1}{K.\gamma}\int_{-\infty}^{v}\frac{\gamma}{f(w|\pmb\nu)}f(w|\pmb\nu)dw=\frac{v}{K}=\text{Unif}(0,k)(v)
	\end{align*}
	The discrete case is analogous to the continuous case and we follow the same proof sketch. 
\end{proof}

\begin{corollary}
	The expected number of sample we need to take is less than  $\frac{n}{K.\gamma}$
\end{corollary}

\begin{proof}
	We know that $\min_{\pmb x\in E}f(\pmb x|\pmb{\nu})\leq\frac{1}{K}$.From the proof of the Theorem \ref{thm:gibbs}, $Pr(v\text{ is accepted})=K.\gamma$, so if we consider each success as independent geometry distribution then the average number is less than $n$-times of largest geometry distribution success, so it is $\frac{n}{K.\gamma}$.
\end{proof}

In \cite{kwon2020algorithm}, it has been proved that if for each pair of Gaussians $\mathcal{N}(\pmb\mu,\sigma\pmb I)$ and $\mathcal{N}(\pmb\mu',\sigma'\pmb I)$, we know their means have distance $\Omega(\max(\sigma,\sigma')\sqrt{\log(\rho_\sigma/\eta_{min})})$ where $\rho_\sigma = \frac{max_{i\in [K]}\sigma_i}{min_{i\in [K]}\sigma_i}$ then with a good parameters initialization, EM algorithm with sample complexity $O(r\eta^{-1}_{min}\log^2(K^2T/\delta)/\epsilon^2)$ can converge to optimal parameters with probability $1-\delta-T/n^{c-2}K^{30}$ where $T=O(\log(1/\epsilon))$ which means 

\begin{equation}
\label{ep-approx}
\forall k\in[K]~\parallel \hat\eta_k^{(T)}-\eta_k^*\parallel_2\leq \eta_k^*\epsilon,~\parallel \hat{\pmb{\mu}}_k^{(T)}-\pmb{\mu}_k^*\parallel_2\leq \sigma_k^*\epsilon,~\parallel (\hat\sigma_k^{(T)})^2-\sigma_k^{*2}\parallel_2\leq \sigma_k^{*2}\epsilon/\sqrt{r}.
\end{equation}

First, we approximate the error of the approximation.

\begin{theorem}\label{thm:samplecomplexity}
	Suppose a mixture of $k$ $d$-dimensional Gaussian's $f$ has parameters such that the separation of the means are $\Omega(c\alpha\max(\sigma,\sigma')\sqrt{\log(\rho_\sigma/\eta_{min})})$ with some given constant $c > 2$ and $\alpha=2.297$. Suppose we use $$n\ge O\bigg(\frac{d\big(\log(K^2T/\delta)[2d+\epsilon+\alpha\beta]\big)^2}{\eta_{min}\alpha^2\epsilon^2}\bigg)$$ samples where $\beta=0.0084$. Then, with a proper initialization, EM algorithm in $T=O(\log(\frac{2d+\epsilon+\alpha\beta}{\alpha\epsilon}))$ iterations approximation $\hat f$, $$TV(f,\hat f)\le\epsilon$$ with probability at least $1-\delta-T/n^{c-2}K^{30}$.
\end{theorem}
\begin{proof}
	For this goal, we know that the area between two normal distribution does not have close form, so we approximate it each normal distribution with Triangular distribution. We use $L^2$ norm for error so we have to solve. Since we have spherical assumption, we can decompose dimensions and solve the optimal point for each dimension independently.
	
	\begin{equation}
	\frac{d}{dx}\bigg[\frac{1}{2\pi}\int_{|x|\ge \alpha}e^{-x^2}dx+\int_{-\alpha}^{\alpha}\bigg(\frac{1-|x/\alpha|}{\alpha}-\frac{e^{-x^2/2}}{\sqrt{2\pi}}\bigg)\bigg]=0
	\label{eq:opt}
	\end{equation}

	so we have $\alpha=2.2975$ with maximum error of $0.042$. Therefore, the best Triangular approximation for the normal distribution $\eta\mathcal{N}(\mu,\sigma)$ is between $[\mu-\beta,\mu+\beta]$ with $\beta=\alpha\sigma/\eta$, So we have to compute the maximum possible error which is the area between two Triangular distributions $[\mu-\frac{\alpha\sigma}{\eta},\mu+\frac{\alpha\sigma}{\eta}]$ and $[\mu+\epsilon\sigma-\frac{\alpha(1-\epsilon)\sigma}{\eta(1+\epsilon)},\mu+\epsilon\sigma+\frac{\alpha(1-\epsilon)\sigma}{\eta(1+\epsilon)}]$. Since the components can be considered independently so triangles separation follows the same well-separation property of the Gaussians. we choose approximation parameters in \ref{ep-approx} (see Fig. \ref{fig:triapprox}) such that the triangular distribution makes the minimum overlap with respect to the triangle of the correct of the normal distribution.
	
	\begin{figure}[]
		\centering
		\includegraphics[width=0.7\columnwidth]{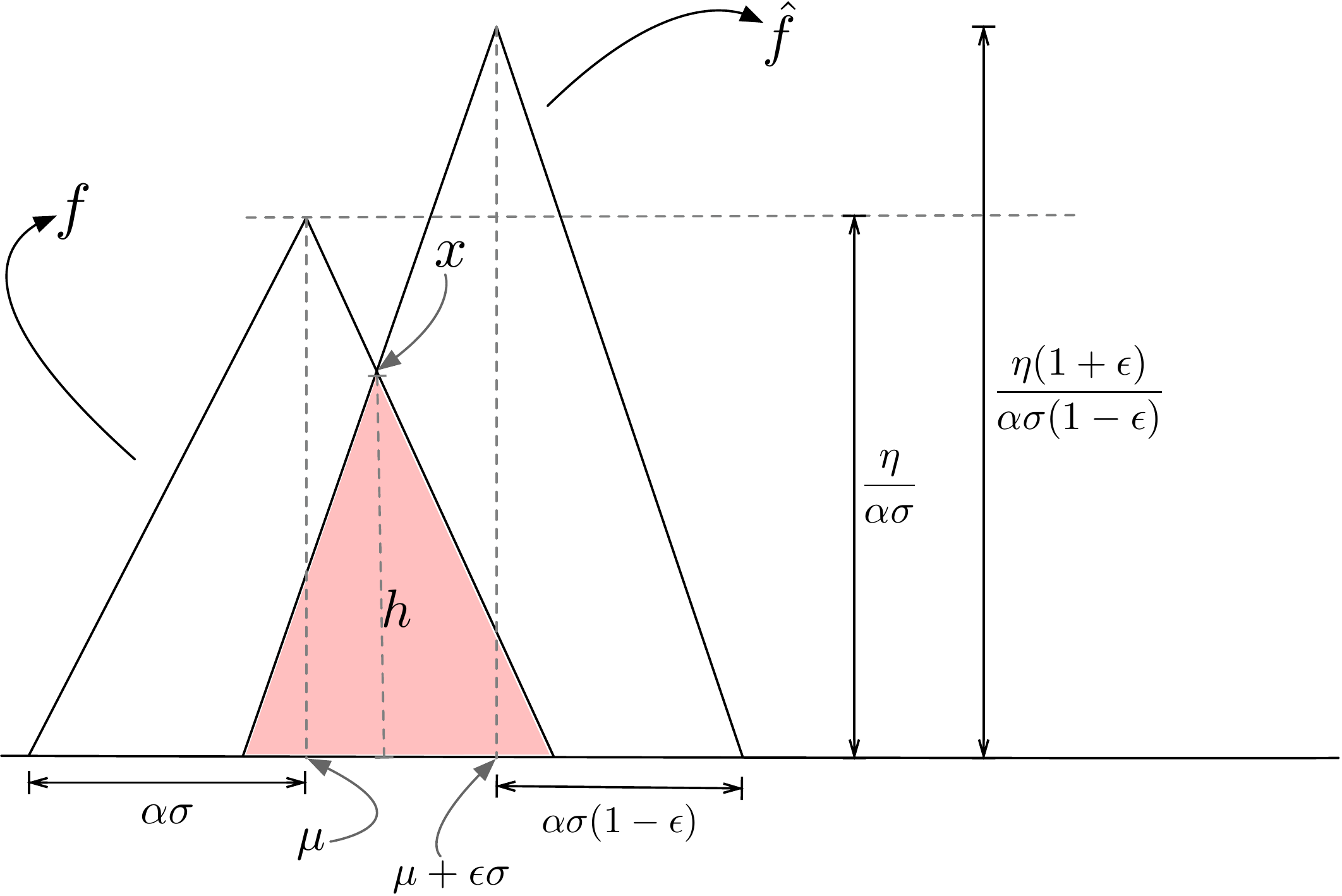}
		\caption{Triangular distribution and its worst-case approximation.}
		\label{fig:triapprox}
	\end{figure}  
	The total variation of $i$-th component is the area that true distribution and its approximation are not overlapped. Therefore, we need to compute the filled area( Fig. \ref{fig:triapprox}), $A_i$, then the error upper-bound is,
	$$error_i\le 2\eta_i+\eta_i\epsilon-2A_i$$
	First we should know that the mean of normal distribution and correspondingly the triangular distribution does not change the error, so we consider $\mu=0$.
	To obtain the point $x$, we need intersect two lines that we can obtain by the given properties of distribution, so we obtain 
	\begin{align*}
	&y_1=\frac{\eta(1+\epsilon)}{\big(\alpha\sigma(1-\epsilon)\big)^2}x-\frac{\epsilon\sigma\eta(1+\epsilon)}{\big(\alpha\sigma(1-\epsilon)\big)^2}+\frac{\eta(1+\epsilon)}{\alpha\sigma(1-\epsilon)}\\&y_2=-\frac{\eta}{(\alpha\sigma)^2}x+\frac{\eta}{\alpha\sigma}
	\end{align*}
	The intersection is,
	$$
	x=\alpha\sigma\bigg(\frac{1+\frac{\epsilon(1+\epsilon)}{\alpha(1-\epsilon)^2}-\frac{1+\epsilon}{1-\epsilon}}{\frac{1+\epsilon}{(1-\epsilon)^2}+1}\bigg)
	$$
	Now, we can compute $h$, the height of overlapped area.
	$$
	h=\frac{\eta}{\alpha\sigma}\bigg[\frac{2+\big(1-\frac{1}{\alpha}\big)\epsilon-\big(1+\frac{1}{\alpha}\big)\epsilon^2}{2-\epsilon+\epsilon^2}\bigg]
	$$
	
	The base of the triangle is,
	\begin{align*}
	b=\big((2-\epsilon)\alpha-\epsilon\big)\sigma
	\end{align*}
	Therefore, the area is,
	\begin{align*}
	A&=\frac{1}{2}bh=\frac{1}{2}\sigma\big((2-\epsilon)\alpha-\epsilon\big).\frac{\eta}{\alpha\sigma}\bigg[\frac{2+\big(1-\frac{1}{\alpha}\big)\epsilon-\big(1+\frac{1}{\alpha}\big)\epsilon^2}{2-\epsilon+\epsilon^2}\bigg]\\&=\frac{1}{2}\eta\big(2-\epsilon-\frac{\epsilon}{\alpha}\big)\bigg[\frac{2+\big(1-\frac{1}{\alpha}\big)\epsilon-\big(1+\frac{1}{\alpha}\big)\epsilon^2}{2-\epsilon+\epsilon^2}\bigg]
	\end{align*}
	We assume $\epsilon^2\approx 0$, so we have,
	\begin{align*}
	A&=\frac{1}{2}\eta\big(2-\epsilon-\frac{\epsilon}{\alpha}\big)\bigg[\frac{2+\big(1-\frac{1}{\alpha}\big)\epsilon}{2-\epsilon}\bigg]=\frac{1}{2}\eta\bigg[2+\big(1-\frac{1}{\alpha}\big)\epsilon-\frac{2\epsilon}{\alpha(2-\epsilon)}\bigg]
	\end{align*}

	Therefore, the overlapped area of component $i$-th is $A_i=\frac{1}{2}\eta_i\big[2+\big(1-\frac{1}{\alpha}\big)\epsilon-\frac{2\epsilon}{\alpha(2-\epsilon)}\big]$. Since we are given a proper separation between components, the total error is the sum of each component error. Therefore, we have,
	\begin{align*}
	TV(f,\hat f)&=\sum_i error_i \le \sum_i 2\eta_i+\eta_i\epsilon-2A_i=2+\epsilon-2\sum_i A_i\\&=2+\epsilon-\sum_i\eta_i\big[2+\big(1-\frac{1}{\alpha}\big)\epsilon-\frac{2\epsilon}{\alpha(2-\epsilon)}\big]\\&=2+\epsilon- \big[2+\big(1-\frac{1}{\alpha}\big)\epsilon-\frac{2\epsilon}{\alpha(2-\epsilon)}\big]=\epsilon\big[\frac{2}{\alpha(2-\epsilon)}+\frac{1}{\alpha}\big]\\&=\frac{4\epsilon}{\alpha(2-\epsilon)}= O(\epsilon).
	\end{align*}
	When $\epsilon$ is small the triangle is an lower bound of the actual overlap between two normal so the error we get is an upper bound.
	For dimension $d$, because our model is spherical, the joint distribution is the product of distribution of each dimension. Therefore, the error of the tale of Gaussian decreasing by increasing the dimension when the dimension is large the data concentrates around the mean. For the right and left tale, if we have $\epsilon$ movement to right, using Taylor expansion, we have $0.0084\epsilon+0.0039\epsilon^2$ for the area between normal and its approximation, so $$TV(f,\hat f)\le d\bigg(\frac{4\epsilon}{\alpha(2-\epsilon)}\bigg)+\beta\epsilon$$
	where $\beta=0.0084$. We determine $\epsilon'=d\big(\frac{4\epsilon}{\alpha(2-\epsilon)}\big)+\beta\epsilon$ so $\epsilon=\frac{\alpha\epsilon'}{2d+\epsilon'+\alpha\beta}$. We replace this into the sample complexity of parameters, $O(d\eta^{-1}_{min}\log^2(K^2T/\delta)/\epsilon^2)$   from \cite{kwon2020algorithm}, the we achieve the result.
\end{proof}

\section{Classical theorems and results}
\begin{theorem}[Vapnik and Chervonenkis \cite{vapnik2015uniform}]
	\label{thm:vceapprox}
	Let $X$ be a domain set and $D$ a probability distribution over $X$. Let $H$ be a class of subsets of $X$ of finite VC-dimension $d$. Let $\epsilon, \delta \in (0,1)$. Let $S \subseteq X$ be picked i.i.d according to $D$ of size $m$. If $m > \frac{c}{\epsilon^2}(d\log \frac{d}{\epsilon}+\log\frac{1}{\delta})$, then  with probability $1-\delta$ over the choice of $S$, we have that $\forall h \in H$
	$$\bigg|\frac{|h\cap S|}{|S|} - P(h)\bigg| < \epsilon$$
\end{theorem}

\begin{theorem}[Concentration inequality for sum of geometric random variables \cite{brown2011wasted}]
	\label{thm:geometricRV}
	Let $X = X_1 + \ldots + X_n$ be $n$ geometrically distributed random variables such that $\mb E[X_i] = \mu$. Then
	$$\mb P[X > (1+\nu)n\mu] \le \exp\bigg(\frac{-\nu^2\mu n}{2(1+\nu)}\bigg)$$
\end{theorem}

\section{Experimental Setup}
 For the first section, we used a random dataset with $1M$ elements of values between $400$ to $1M$. We have applied two approaches for the generative process to produce duplication, uniform and arbitrary distributions. For arbitrary distributions, the frequency of entities chosen randomly in the range, so we have a freedom to evaluate in the estimator in variety of configurations. We determine the error as $Error = \vert RealAvg-EstimatedAvg\vert/RealAvg$.
We use two real datasets which they are publicly available.

\textbf{TPC-H Dataset\footnote{http://www.tpc.org/tpch}} It contains 1.5GB TPC-H benchmark3 dataset (8,609,880 Records in lineitem table). The line item table schema simulates industrial purchase order records. We used this dataset to model errors where the purchase orders were digitized using optical character recognition (OCR). We similar to Sample-and-Clean \cite{wang2014sample} randomly duplicated 20\% of tuples with the following distribution: 80\%  one duplicate, 15\%  two duplicates, 5\%  three duplicates.

\textbf{Sensor Dataset\footnote{http://db.csail.mit.edu/labdata/labdata.html}} We also applied our approach to a dataset of indoor temperature, humidity, and light sensor readings in the Intel Berkeley Research Lab. The dataset is publicly available for data cleaning and sensor network research from MIT CSAIL5.

\textbf{Publications Dataset} This dataset is a real-world bibliographical information of scientific publications \cite{draisbach2010dude}. The dataset has 1,879 publication records with duplicates. The ground truth of duplicates is available. To perform clustering on this dataset we first tokenized each
publication record and extracted 3-grams from them. Then, on 3-grams we used Jaccard distance to define distance between two records.

\textbf{E-commerce products I\footnote{https://dbs.uni-leipzig.de/en}} This dataset contains $1,363$ products from Amazon, and $3,226$ products from Google, and the ground truth has $1,300$
matching products. 

\textbf{E-commerce products II\footnote{ https://dbs.uni-leipzig.de/en/research/projects/object\_matching/fever/benchmark\_datasets\_for\_entity\_resolution}} This dataset contains 1,082 products from Abt, and 1,093 products from Buy, and the ground truth has 1,098 matching products.

\textbf{Restaurants Dataset\footnote{ http://www.cs.utexas.edu/users/ml/riddle/data.html}} The fifth dataset is a list of 864 restaurants from the Fodor’s and Zagat’s restaurant guides that contains 112 duplicates.

\end{appendix}
\bibliography{sampling_deduplication}
\end{document}